\def\year{2018}\relax
\documentclass[letterpaper]{article}
\usepackage{paralist}
\usepackage{aaai18}
\usepackage{times}
\usepackage{helvet}
\usepackage{courier}
\usepackage{url}
\usepackage{graphicx}
\usepackage{amsmath}
\usepackage{amsthm}
\usepackage{thm-restate}
\usepackage{algorithm}
\usepackage[noend]{algpseudocode}
\usepackage{chngcntr}
\usepackage{macros}
\usepackage{pgfplots}
\usepackage{etoolbox}
\usepgfplotslibrary{statistics}
\usepgfplotslibrary{groupplots}
\pgfplotsset{compat=1.12}
\usepackage{adjustbox}
\usepackage{multirow}
\allowdisplaybreaks
\newtoggle{withappendix}
\toggletrue{withappendix}

\begin{document}

\title{Goal-Driven Query Answering for Existential Rules with Equality}

\author{
   Michael Benedikt \and Boris Motik \\ University of Oxford
   \And Efthymia Tsamoura \\ Alan Turing Institute \& University of Oxford
}

\maketitle

\begin{abstract}
Inspired by the magic sets for Datalog, we present a novel goal-driven approach
for answering queries over terminating existential rules with equality (aka
TGDs and EGDs). Our technique improves the performance of query answering by
pruning the consequences that are not relevant for the query. This is
challenging in our setting because equalities can potentially affect all
predicates in a dataset. We address this problem by combining the existing
singularization technique with two new ingredients: an algorithm for
identifying the rules relevant to a query and a new magic sets algorithm. We
show empirically that our technique can significantly improve the performance
of query answering, and that it can mean the difference between answering a
query in a few seconds or not being able to process the query at all.

\end{abstract}

\section{Introduction}

\emph{Existential rules with equality}, also known as \emph{tuple- and equality
generating dependencies} (TGDs and EGDs) or Datalog$^\pm$ rules, extend Datalog
by allowing rule heads to contain existential quantifiers and the equality
predicate $\equals$. Answering a conjunctive query $\predQ$ over a set of
existential rules $\Sigma$ and a base instance $B$ is key to dealing with
incomplete information in information systems \cite{Fagin2005a}. The problem is
undecidable in general, but many decidable cases are known
\cite{DBLP:journals/ai/BagetLMS11,DBLP:journals/semweb/KonigLMT15,DBLP:conf/ijcai/BagetBMR15,DBLP:conf/ijcai/GottlobMP15,DBLP:conf/kr/LeoneMTV12}.
Systems such as Llunatic \cite{llunatic}, RDFox
\cite{mnpho14parallel-materialisation-RDFox}, DLV$^\exists$
\cite{DBLP:conf/kr/LeoneMTV12}, ChaseFUN \cite{DBLP:conf/edbt/BonifatiIL17},
Ontop \cite{DBLP:journals/semweb/CalvaneseCKKLRR17}, and Graal
\cite{DBLP:conf/ruleml/BagetLMRS15} implement various query answering
techniques. One solution to this problem is to evaluate the query in a
\emph{universal model} of ${\Sigma \cup B}$, and a common and practically
relevant case is when a finite universal model can be computed using a
\emph{chase} procedure. Many chase variants have been proposed. Although
checking chase termination is undecidable for all variants
\cite{chaserevisited,marnette2009}, numerous sufficient \emph{acyclicity}
conditions
\cite{marnette2009,DBLP:conf/ijcai/KrotzschR11,ghkkmmw13acyclicity-journal}
guarantee termination of at least the oblivious Skolem chase; we call such
${\Sigma \cup B}$ \emph{terminating}.

Computing a universal model in full when only a specific query is to be
answered may be inefficient. We experimentally show that query answers often
depend only on a small subset of the universal model, particularly for queries
containing constants, so the chase may perform a lot of unnecessary work.
Moreover, universal models sometimes cannot be computed due to their size. In
such cases, \emph{goal-driven} query answering techniques, which take the query
into account, hold the key to efficient query answering.

One possibility, implemented in systems such as Ontop
\cite{DBLP:journals/semweb/CalvaneseCKKLRR17} and Graal
\cite{DBLP:conf/ruleml/BagetLMRS15}, is to \emph{rewrite} the relevant rules
into a new query that can be evaluated directly on the base instance. Rewriting
into first-order queries is possible for DL-lite
\cite{DBLP:journals/jar/CalvaneseGLLR07}, linear TGDs
\cite{DBLP:journals/ws/CaliGL12}, and sticky TGDs
\cite{DBLP:journals/tods/GottlobOP14,DBLP:conf/rr/CaliGP10}, among others.
Frontier-guarded \cite{bagetdatalogrewriting}, weakly-guarded
\cite{DBLP:conf/pods/GottlobRS14}, and Horn-$\mathcal{SHIQ}$
\cite{DBLP:conf/aaai/EiterOSTX12} rules can be rewritten into Datalog.
Rewriting approaches, however, cannot handle common properties that can be
handled via the chase, such as transitivity, and they typically support only
``innocuous'' equalities that do not affect query answers.

In Datalog and logic programming, the \emph{magic sets} algorithm
\cite{DBLP:conf/pods/BancilhonMSU86,DBLP:journals/jlp/BeeriR91} annotates the
rules with \emph{magic} atoms, which ensure that bottom-up evaluation of the
magic program simulates top-down query evaluation. This influential idea has
been adapted to disjunctive \cite{DBLP:journals/ai/AlvianoFGL12} and finitely
recursive \cite{DBLP:conf/lpnmr/CalimeriCIL09} programs, programs with
aggregates \cite{DBLP:conf/lpnmr/AlvianoGL11}, and \emph{Shy} existential rules
\cite{DBLP:conf/datalog/AlvianoLMTV12}. These approaches, however, do not
handle existential rules with equality.

In this paper we present what we believe to be the first goal-driven query
answering technique for terminating existential rules \emph{with} equality.
Given a set of rules $\Sigma$ and a query $\predQ$, we compute a logic program
$P$ such that, for each base instance $B$, the answers to $\predQ$ on ${\Sigma
\cup B}$ and ${P \cup B}$ coincide, but processing the latter is typically much
more efficient. Our approach combines existing techniques such as
\emph{singularization} \cite{marnette2009} with two new ingredients: a new
\emph{relevance analysis} algorithm that identifies irrelevant rules, and a new
magic sets variant that handles existential rules with equality. These two
techniques are complementary: the first one prunes rules whose consequences are
irrelevant to the query, and the second one prunes the irrelevant consequences
of the remaining rules. Since equalities can potentially affect any predicate,
both techniques are needed to efficiently identify the relevant equalities.

We have empirically evaluated our technique on a recent benchmark that includes
a diverse set of existential rules \cite{chasebench}. Our results show that
goal-driven query answering is generally more efficient than computing the
chase in full. In fact, our approach can mean the difference between success
and failure: even though the chase cannot be computed in several cases, we can
answer the relevant queries in a few seconds. We also show that relevance
analysis alone is very effective at eliminating irrelevant rules even without
equalities. Finally, we show that magic sets alone can be less efficient on
queries without constants, but it greatly benefits queries with constants. A
combination of both techniques usually provides the best performance.

All proofs and further experimental results are given in the
\iftoggle{withappendix}{appendix}{extended version \cite{extended-version}}.

\section{Preliminaries} \label{sec:preliminaries}

We use the standard first-order logic notions of \emph{variables},
\emph{constants}, \emph{function symbols}, \emph{predicates}, \emph{arity},
\emph{terms}, \emph{atoms}, and \emph{formulas}, and $\equals$ is the binary
\emph{equality} predicate. Atoms ${{\equals}(s,t)}$ are \emph{equational} and
are usually written as ${s \equals t}$, and all other atoms are
\emph{relational}. A \emph{fact} is a variable-free atom, an \emph{instance} is
a (possibly infinite) set of facts, and a \emph{base instance} is a finite,
function-free instance. We consider two notions of entailment: $\models$
interprets $\equals$ as an ``ordinary predicate'' without any special
semantics, whereas $\modelsEq$ interprets $\equals$ under the usual semantics
of equality without the \emph{unique name assumption} (UNA)---that is, distinct
constants can be derived equal. A theory $T$ \emph{satisfies UNA} if no two
distinct constants $a$ and $b$ exist such that ${T \modelsEq a \equals b}$. For
example, let ${\varphi = A(a) \wedge a \equals b}$. Then, ${\varphi \modelsEq
A(b)}$ and ${\varphi \modelsEq b \equals a}$, and $\varphi$ does not satisfy
UNA. In contrast, ${\varphi \not\models A(b)}$ and ${\varphi \not\models b
\equals a}$. We often abbreviate a tuple ${t_1, \dots, t_n}$ as ${\vec t}$, and
we often treat ${\vec t}$ as a set and write ${t_i \in \vec t}$.

A term $t$ \emph{occurs} in a term, atom, tuple, or set $X$ if $X$ contains $t$
possibly nested inside another term; $\vars{X}$ is the set of variables
occurring in $X$; and $X$ is \emph{ground} if ${\vars{X} = \emptyset}$. For
$\sigma$ a mapping of variables and/or constants to terms, $\sigma(X)$ replaces
each occurrence of a term $t$ in $X$ with $\sigma(t)$ if the latter is defined,
and $\sigma$ is a \emph{substitution} if its domain is finite and contains only
variables. For $\mu$ a mapping of ground terms to ground terms,
$\tapply{\mu}{X}$ replaces each occurrence of a term $t$ not nested in a
function symbol with $\mu(t)$ if the latter is defined. For example, let ${A =
R(f(x),g(a))}$; then, ${\sigma(A) = R(f(b),g(c))}$ for ${\sigma = \{ x \mapsto
b, a \mapsto c \}}$, and ${\tapply{\mu}{A} = R(f(x),h(d))}$ for ${\mu = \{ a
\mapsto b, g(a) \mapsto h(d) \}}$.

\emph{Existential rules} are logical implications of two forms: ${\forall \vec
x. [\lambda(\vec x) \rightarrow \exists \vec y. \rho(\vec x,\vec y)]}$ is a
\emph{tuple-generating dependency} (TGD), and ${\forall \vec x. [\lambda(\vec
x) \rightarrow t_1 \equals t_2]}$ is an \emph{equality-generating dependency}
(EGD), where ${\lambda(\vec x)}$ and ${\rho(\vec x,\vec y)}$ are conjunctions
of relational, function-free atoms with variables in ${\vec x}$ and ${\vec x
\cup \vec y}$, respectively, $t_1$ and $t_2$ are variables from ${\vec x}$ or
constants, and each variable in ${\vec x}$ occurs in ${\lambda(\vec x)}$.
Quantifiers ${\forall \vec x}$ are commonly omitted. Conjunction ${\lambda(\vec
x)}$ is the \emph{body} of a rule, and ${\rho(\vec x,\vec y)}$ and ${t_1
\equals t_2}$ are its \emph{head}. We assume that queries are defined using a
\emph{query predicate} $\predQ$ that does not occur in rule bodies or under
existential quantifiers. A tuple ${\vec a}$ of constants is an \emph{answer} to
$\predQ$ on a finite set of existential rules $\Sigma$ and a base instance $B$
iff ${\Sigma \cup B \modelsEq \predQ(\vec a)}$.

When treating $\equals$ as ``ordinary,'' we allow rule bodies to contain
equality atoms, and we can axiomatize the ``true'' semantics of $\equals$ for
$\Sigma$ as follows. Let $\Ref{\Sigma}$ and $\Cong{\Sigma}$ contain the
\emph{reflexivity} \eqref{eq:reflexivity} and \emph{congruence}
\eqref{eq:congruence} axioms, respectively, instantiated for each $n$-ary
predicate $R$ in $\Sigma$ distinct from $\equals$ and each ${1 \leq i \leq n}$.
Let $\SymTrans$ contain the \emph{symmetry} \eqref{eq:symmetry} and the
\emph{transitivity} \eqref{eq:transitivity} axioms. We assume that each base
instance contains only the predicates of $\Sigma$, since the equality axioms
are then determined only by $\Sigma$. Then, for each base instance $B$ and
tuple of constants ${\vec a}$, we have ${\Sigma \cup B \modelsEq \predQ(\vec
a)}$ if and only if ${\Sigma \cup \Ref{\Sigma} \cup \Cong{\Sigma} \cup
\SymTrans \cup B \models \predQ(\vec a)}$.
\begin{align}
    R(\dots, x_i, \dots)                            & \rightarrow x_i \equals x_i       \label{eq:reflexivity} \\
    R(\dots, x_i, \dots) \wedge  x_i \equals x_i'   & \rightarrow R(\dots, x'_i, \dots) \label{eq:congruence} \\
    y \equals x                                     & \rightarrow x \equals y           \label{eq:symmetry} \\
    x \equals y \wedge y \equals z                  & \rightarrow x \equals z           \label{eq:transitivity}
\end{align}

Our algorithms use logic programming, which we define next. A \emph{rule} $r$
has the form ${R(\vec t) \leftarrow R_1(\vec t_1) \wedge \dots \wedge R_n(\vec
t_n)}$, where $R(\vec t)$ and $R_i(\vec t_i)$ are atoms possibly containing
function symbols. Each variable in $r$ must occur in some ${\vec t_i}$. To
distinguish existential from logic programming rules, we use $\rightarrow$ for
the former and $\leftarrow$ for the latter. Conjunction ${\body{r} = R_1(\vec
t_1) \wedge \dots \wedge R_n(\vec t_n)}$ is the \emph{body} of $r$ and we often
treat it as a set, and atom ${\head{r} = R(\vec t)}$ is the \emph{head} of $r$.
Predicate $\equals$ is always ordinary in logic programming, so $R$ and $R_i$
can be $\equals$. A \emph{(logic) program} $P$ is a finite set of rules, and it
is interpreted in first-order logic as usual. Again, we assume that a query in
$P$ is defined using the predicate $\predQ$ not occurring in rule bodies. For
$I$ an instance, $\consOf{P}{I}$ is the result of extending $I$ with
$\sigma(\head{r})$ for each rule ${r \in P}$ and substitution $\sigma$ such
that ${\sigma(\body{r}) \subseteq I}$. Finally, for $B$ a base instance, we
inductively define a sequence of interpretations where ${I_0 = B}$ and ${I_i =
\consOf{P}{I_{i-1}}}$ for ${i > 0}$; then, the \emph{least fixpoint} of $P$ on
$B$ is ${\fixpoint{P}{B} = \bigcup_{i \geq 0} I_i}$. It is well known that ${P
\cup B \models F}$ iff ${F \in \fixpoint{P}{B}}$ holds for each fact $F$.

Our algorithms reduce query answering over existential rules to reasoning in
logic programming. We eliminate existential quantifiers by computing the
\emph{Skolemization} $\sk{\Sigma}$ of a set $\Sigma$ of existential rules. Set
$\sk{\Sigma}$ contains each EGD of $\Sigma$ as a logic programming rule and,
for each TGD ${\tau = \lambda(\vec x) \rightarrow \exists \vec y.\rho(\vec
x,\vec y) \in \Sigma}$ and each ${R(\vec t) \in \rho(\vec x,\vec y)}$, set
$\sk{\Sigma}$ contains the rule ${\sigma(R(\vec t)) \leftarrow \lambda(\vec
x)}$ where $\sigma$ is a substitution mapping each variable ${y \in \vec y}$ to
${f_{\tau,y}(\vec x')}$ for ${\vec x' = \vars{\lambda(\vec x)} \cap
\vars{\rho(\vec x, \vec y)}}$ and $f_{\tau,y}$ a fresh function symbol unique
for $\tau$ and $y$. Let ${P = \sk{\Sigma}}$; if $\Sigma$ and $B$ do not contain
$\equals$, then for each predicate $R$ and tuple ${\vec a}$ of constants, we
have ${\Sigma \cup B \models R(\vec a)}$ iff ${P \cup B \models R(\vec a)}$. If
$\Sigma$ or $B$ contains $\equals$, we can axiomatize equality using axioms
$\Ref{P}$, $\Cong{P}$, and $\SymTrans$ defined analogously to
\eqref{eq:reflexivity}--\eqref{eq:transitivity}; then, ${P' = P \cup \Ref{P}
\cup \Cong{P} \cup \SymTrans}$ captures the intended semantics of $\equals$,
and ${\Sigma \cup B \modelsEq R(\vec a)}$ iff ${P' \cup B \models R(\vec a)}$.

Now let $P$ and $P'$ be as in the previous paragraph. We could answer queries
over such $P'$ by computing $\fixpoint{P'}{B}$ and evaluating $\predQ$ on it,
but this is inefficient even when $P$ is just a Datalog program
\cite{mnph15owl-sameAs-rewriting} since firing congruence rules can be
prohibitively expensive. The \emph{chase for logic programs} offers a more
efficient method for reasoning with ${P' \cup B}$ by efficiently computing a
representation of $\fixpoint{P'}{B}$. It is applicable if $P$ does not contain
constants, function symbols, or $\equals$ in the rule bodies. The algorithm
constructs a sequence of pairs ${\langle I_i, \mu_i \rangle}$, ${i \geq 0}$,
where $I_i$ is an instance and $\mu_i$ maps ground terms to ground terms. The
algorithm initializes $I_0$ to a normalized version of $B$ where all constants
reachable by $\equals$ in $B$ are replaced by a representative, and it records
these replacements in $\mu_0$. For each ${i > 0}$, the chase selects a rule ${r
\in P}$ and a substitution $\sigma$ with ${\sigma(\body{r}) \subseteq I_{i-1}}$
and (i)~if $\sigma(\head{r})$ is of the form ${s \equals t}$ and ${s \neq t}$,
then one term, say $s$, is selected as the \emph{representative}, and $I_i$ and
$\mu_i$ are obtained from $I_{i-1}$ and $\mu_{i-1}$ by replacing $t$ with $s$
and setting ${\mu_i(t) = s}$; and (ii)~if $\sigma(\head{r})$ does not contain
$\equals$ and ${\sigma(\head{r}) \not \in I_{i-1}}$, then ${\mu_i = \mu_{i-1}}$
and ${I_i = I_{i-1} \cup \{ \tapply{\mu_i}{\sigma(\head{r})} \}}$. The
computation proceeds until no rule is applicable and then returns the final
pair ${\langle I_n, \mu_n \rangle}$. If the representatives are always chosen
as smallest in an arbitrary, but fixed well-founded order on ground terms, then
the result is unique for $P$ and $B$ and it is called the \emph{chase} of $P$
on $B$, written $\chase{P}{B}$. The following properties of the chase are well
known \cite{chasebench}.

\begin{proposition}\label{prop:chasecorrect}
    For each program $P$, base instance $B$, ${\chase{P}{B} = \langle I, \mu
    \rangle}$, and ${P' = P \cup \Ref{P} \cup \Cong{P} \cup \SymTrans}$,
    (i)~${\mu(t_1) = \mu(t_2)}$ if and only if ${P' \cup B \models t_1 \equals
    t_2}$, for all ground terms $t_1$ and $t_2$, and (ii)~${P' \cup B \models
    R(\vec t)}$ if and only if ${R(\tapply{\mu}{\vec t}) \in I}$, for each
    ground relational atom ${R(\vec t)}$.
\end{proposition}

Intuitively, $\mu(t)$ is a unique \emph{representative} of each ground term
$t$, and the chase maintains and propagates facts only among the representative
facts of $\fixpoint{P'}{B}$, instead of na{\"i}vely firing congruence rules.
The algorithm is used in systems such as Llunatic and RDFox \cite{chasebench}.

When reasoning with equality, an important question is whether programs are
allowed to equate constants. We say that $P$ and $B$ \emph{satisfy UNA} if ${P'
\cup B \models a \equals b}$ implies ${a = b}$. Our algorithms do not require
UNA to be satisfied, but certain steps can be optimized if we know that UNA is
satisfied.

\section{Motivation and Overview}\label{sec:motivation}

To understand the challenges of goal-driven query answering over existential
rules with $\equals$, let $\exSigma$ consist of
\eqref{runex:Q}--\eqref{runex:egd-T}.
\begin{align}
    A(x) \wedge R(x,y)                      & \rightarrow \predQ(x)                     \label{runex:Q} \\
    S(x,z)                                  & \rightarrow \exists y. R(x,y)             \label{runex:R} \\
    R(x,y) \wedge S(x,x') \wedge R(x',y')   & \rightarrow y \equals y'                  \label{runex:egd-RS} \\
    B(x)                                    & \rightarrow \exists y. T(x,y) \wedge A(y) \label{runex:TA} \\
    T(x,y)                                  & \rightarrow x \equals y                   \label{runex:egd-T}
\end{align}
Let ${\exB = \{ B(a_1) \} \cup \{ S(a_{i-1},a_i) \mid 1 < i \leq n \}}$. One
can check that ${\exSigma \cup \exB \models \predQ(a_i)}$ holds only for ${i =
1}$; however, all bottom-up techniques known to us will ``fire''
\eqref{runex:R} and \eqref{runex:egd-RS} for all $a_i$. In logic programming,
goal-driven or top-down approaches, such as SLD resolution, start from the
query and search for proofs backwards. The magic sets algorithm transforms a
program so that evaluating the result bottom-up mimics top-down evaluation.
These approaches are not directly applicable to existential rules, but we can
apply them to the program ${P' = P \cup \Ref{P} \cup \Cong{P} \cup \SymTrans}$,
obtained by Skolemizing TGDs as ${P = \sk{\exSigma}}$ and then axiomatizing
equality. This, however, is inefficient since the congruence axioms introduce
many redundant proofs. In particular, Skolemizing \eqref{runex:R} produces
${R(x,f(x)) \leftarrow S(x,z)}$. By rule \eqref{runex:egd-RS}, we have ${P'
\models f(a_{i-1}) \equals f(a_i)}$ for ${1 < i \leq n}$ so, by the
reflexivity, symmetry, and transitivity axioms for $\equals$, we have ${P'
\models f(a_i) \equals f(a_j)}$ for ${1 \leq i,j \leq n}$. Hence, by the
congruence axioms, we have ${P' \models R(a_i,f(a_j))}$. Thus, ${P' \models
\predQ(a_1)}$ has (at least) $n$ proofs, where the first step uses a ground
rule instance ${\predQ(a_1) \leftarrow A(a_1) \wedge R(a_1,f(a_i))}$ for each
${1 \leq i \leq n}$. The magic sets algorithm will explore all of these proofs,
which is very expensive. In contrast, our technique can answer the query by
considering this rule instantiated only for ${i = 1}$ (see
Example~\ref{ex:final}).

We present an approach that gives the benefits of top-down approaches, while
radically pruning the set of considered proofs. Our approach has additional
benefits. It does not require UNA, but certain steps can be optimized if
${\Sigma \cup B}$ satisfies UNA (e.g., if an earlier UNA check succeeded).
Moreover, it preserves chase termination, and it includes an optimized magic
set transformation using the symmetry of equality to greatly reducing the
number of output rules.

Our technique is presented in the pipeline shown in
Algorithm~\ref{alg:answer-query}. Instead of axiomatizing equality, we first
apply \emph{singularization} (line~\ref{alg:answer-query:sg}), a well-known
transformation that makes all relevant equalities explicit
\cite{marnette2009,tencate2009}, and then we convert the result to a logic
program using Skolemization (line~\ref{alg:answer-query:sk}). Next, we apply a
relevance analysis algorithm (line~\ref{alg:answer-query:relevance}) that
identifies the rules relevant to the query. We next apply the magic sets
transformation optimized for $\equals$ (line~\ref{alg:answer-query:magic}); the
removal of irrelevant equality atoms during relevance analysis ensures that
this step produces a smaller program. Finally, we remove the function symbols
(line~\ref{alg:answer-query:defun}) and equalities
(line~\ref{alg:answer-query:desg}) from rule bodies, obtaining a program that
can be safely evaluated using the chase for logic programs
(lines~\ref{alg:answer-query:chase}--\ref{alg:answer-query:predQ:end}). We
explain the components in detail in the following sections.

\begin{algorithm}[tp]
\caption{Compute the answers to query $\predQ$ over a finite set of existential rules $\Sigma$ and a base instance $B$}\label{alg:answer-query}
\begin{small}
\begin{algorithmic}[1]
    \State $\Sigma_1 \defeq \sg{\Sigma}$                                                                \label{alg:answer-query:sg}
    \State $P_2 \defeq \sk{\Sigma_1}$                                                                   \label{alg:answer-query:sk}
    \State $P_3 \defeq \relevance{P_2}{B}$                                                              \label{alg:answer-query:relevance}
    \State $P_4 \defeq \magic{P_3}$                                                                     \label{alg:answer-query:magic}
    \State $P_5 \defeq \defun{P_4}$                                                                     \label{alg:answer-query:defun}
    \State $P_6 \defeq \desg{P_5}$                                                                      \label{alg:answer-query:desg}
    \State $\langle I, \mu \rangle \defeq \chase{P_6}{B}$                                               \label{alg:answer-query:chase}
    \For{\textbf{each} $\predQ(\vec a) \in I$ where $\vec a$ are constants}                             \label{alg:answer-query:predQ:start}
        \State \textbf{output} each tuple of constants $\vec b$ with $\tapply{\mu}{\vec b} = \vec a$    \label{alg:answer-query:output}
    \EndFor                                                                                             \label{alg:answer-query:predQ:end}
\end{algorithmic}
\end{small}
\end{algorithm}

\section{Singularization}\label{sec:singularization}

Singularization is an alternative to congruence axioms.

\begin{definition}\label{def:sg}
    A \emph{singularization} of an existential rule $\tau$ is obtained from
    $\tau$ by exhaustively (i)~replacing each occurrence of a constant $c$ in a
    relational body atom with a fresh variable $x$ and adding atom ${x \equals
    c}$ to the body, and (ii)~for each variable $x$ occurring at least twice in
    (not necessarily distinct) relational body atoms, replacing one such
    occurrence with a fresh variable $x'$ and adding atom ${x' \equals x}$ to
    the body.

    A \emph{singularization} of a set of existential rules $\Sigma$ defining
    the query predicate $\predQ$ is obtained by replacing each TGD of the form
    ${\varphi \rightarrow \predQ(x_1, \dots, x_n)}$ with \eqref{eq:sing:predQ}
    for ${x_1', \dots, x_n'}$ fresh variables, and then singularizing all
    existential rules.
    \begin{align}
        \varphi \wedge \textstyle\bigwedge_{i=1}^n x_i \equals x_i' \rightarrow \predQ(x_1', \dots, x_n')   \label{eq:sing:predQ}
    \end{align}
\end{definition}

\begin{example}
On $\exSigma$, singularization leaves \eqref{runex:R}, \eqref{runex:TA}, and
\eqref{runex:egd-T} intact since their bodies do not contain repeated
variables. Rules \eqref{runex:Q} and \eqref{runex:egd-RS} are singularized as
\eqref{runex:Q-sg} and \eqref{runex:egd-RS-sg}.
\begin{align}
    A(x'') \wedge x \equals x'' \wedge R(x,y) \wedge x \equals x'   & \rightarrow \predQ(x') \label{runex:Q-sg} \\
    \begin{array}{@{}l@{}}
        R(x,y) \wedge x \equals x'' \; \wedge \\
        \qquad S(x'',x') \wedge x' \equals x''' \wedge R(x''',y') \\
        \end{array}                                                 &
        \begin{array}{@{}l@{}}
            \\
            \; \rightarrow y \equals y'
    \end{array}                                                                              \label{runex:egd-RS-sg}
\end{align}
\end{example}
The result of singularization is not unique: \eqref{runex:Q} could also produce
${A(x) \wedge x \equals x'' \wedge R(x'',y) \wedge x \equals x' \rightarrow
\predQ(x')}$. In our approach, we let $\sg{\Sigma}$ be any singularization of
$\Sigma$.

Singularization highlights the relevant equalities originating from joins,
which in turn preserves all query answers without relying on congruence axioms:
for each base instance $B$ and tuple ${\vec a}$ of constants, we have ${\Sigma
\cup B \modelsEq \predQ(\vec a)}$ if and only if ${\sg{\Sigma} \cup
\Ref{\Sigma} \cup \SymTrans \cup B \models \predQ(\vec a)}$. Singularization
still relies on reflexivity axioms. As an important optimization, we prove that
these do not need to be analyzed in the remaining steps of our pipeline, which
ensures that our pipeline produces smaller, more efficient programs. To achieve
this, we show that our transformations produce rules satisfying the following
condition.

\begin{definition}\label{def:eq-safe}
    A rule $r$ is \emph{$\equals$-safe} if, for each equality atom ${A \in
    \body{r}}$, the atom is of the form ${x \equals y}$ or ${x \equals s}$ for
    $s$ a ground term, and ${\vars{A} \cap \vec t_i \neq \emptyset}$ for some
    relational atom ${R_i(\vec t_i) \in \body{r}}$. A program is
    \emph{$\equals$-safe} all its rules are $\equals$-safe.
\end{definition}

Intuitively, $\equals$-safety ensures that each fact ${t \equals t}$ matching a
body atom of a rule $r$ can be derived from another relational body atom of
$r$. Thus, we do not need to pass the reflexivity axiom as input to the steps
of our pipeline in order to determine which of these are pertinent to $\predQ$.
Instead, the pertinent reflexivity axioms are determined directly by the
predicates occurring in the result of each pipeline step.

We apply Skolemization after singularization to eliminate existential
quantifiers.

\begin{example}
In our running example, only \eqref{runex:R} and \eqref{runex:TA} contain
existential quantifiers, so they are replaced by \eqref{runex:R-sk}, and
\eqref{runex:T-sk} and \eqref{runex:A-sk}; all other rules are reinterpreted as
logic programming rules. Program $P_2$ contains rules
\eqref{runex:Q-sk}--\eqref{runex:egd-T-sk}.
\begin{align}
    \predQ(x')      & \leftarrow A(x'') \wedge x \equals x'' \wedge R(x,y) \wedge x \equals x'                          \label{runex:Q-sk} \\
    R(x,f(x))       & \leftarrow S(x,z)                                                                                 \label{runex:R-sk} \\
    \begin{array}{@{}l@{}}
        y \equals y' \\
        \\
    \end{array}     & \begin{array}{@{\;}l@{}}
                        \leftarrow R(x,y) \wedge x \equals x'' \wedge S(x'',x') \; \wedge \\
                        \hspace{2.5cm} x' \equals x''' \wedge R(x''',y') \\
                       \end{array} \label{runex:egd-RS-sk} \\
    T(x,g(x))       & \leftarrow B(x)                                                                                   \label{runex:T-sk} \\
    A(g(x))         & \leftarrow B(x)                                                                                   \label{runex:A-sk} \\
    x \equals y     & \leftarrow T(x,y)                                                                                 \label{runex:egd-T-sk}
\end{align}

The answers to $\predQ$ on ${\Sigma \cup B}$ and ${P_2 \cup \Ref{P_2} \cup
\SymTrans \cup B}$ coincide on each base instance $B$, but the absence of
congruence axioms considerably reduces the number of proofs: EGD
\eqref{runex:egd-RS} still ensures ${P_2 \models f(a_i) \equals f(a_j)}$ for
${1 \leq i,j \leq n}$, but ${P_2 \models R(a_i,f(a_j))}$ holds only for ${i =
j}$. Thus, the only remaining proof of ${P_2 \models \predQ(a_1)}$ is via the
ground instance ${\predQ(a_1) \leftarrow C(a_1) \wedge R(a_1,f(a_1))}$, which
benefits all goal-driven techniques, including magic sets. Also, none of the
facts derived by EGD \eqref{runex:egd-RS} contribute to a proof of
$\predQ(a_1)$, so singularization makes EGD redundant; in
Section~\ref{sec:relevance} we present way to detect and eliminate such rules.
\end{example}

\section{Relevance Analysis}\label{sec:relevance}

The next step of the pipeline eliminates rules all of whose consequences are
irrelevant to $\predQ$. The idea is to homomorphically embed $B$ into a much
smaller instance $B'$ called an \emph{abstraction} of $B$. If $B'$ is
sufficiently small, we can analyze ways to derive answers to $\predQ$ on $B'$;
since homomorphism composition is a homomorphism, this will uncover all ways to
derive an answer to $\predQ$ on the original base instance $B$.

\begin{definition}\label{def:abstraction}
    A base instance $B'$ is an \emph{abstraction} of a base instance $B$
    w.r.t.\ a program $P$ if there exists a homomorphism $\eta$ from $B$ to
    $B'$ preserving the constants in $P$---that is, $\eta$ maps constants to
    constants such that ${\eta(B) \subseteq B'}$ and ${\eta(c) = c}$ for each
    constant $c$ occurring in $P$.
\end{definition}

To abstract $B$ into $B'$, we can use the \emph{critical instance} for $B$: for
$C$ the set of constants of $P$ and $\ast$ a fresh constant, we let $B'$
contain ${R(\vec a)}$ for each $n$-ary predicate $R$ occurring in $B$ and each
${\vec a \in (C \cup \{ \ast \})^n}$. We can further refine the abstraction if
predicates are \emph{sorted} (i.e., each predicate position is associated with
a sort such as strings or integers): we introduce a distinct fresh constant
$\ast_i$ per sort, and we form $B'$ as above while honoring the sorting
requirements.

Algorithm~\ref{alg:relevance} takes an $\equals$-safe program $P$ and a base
instance $B$, and it returns the rules relevant to answering $\predQ$ on $B$.
It selects an abstraction $B'$ of $B$ w.r.t.\ $P$
(line~\ref{alg:relevance:abstraction}), computes the consequences $I$ of $P$ on
$B'$ (line~\ref{alg:relevance:abstraction-fixpoint}), and identifies the rules
of $P$ contributing to the answers of $\predQ$ on $B'$ by a form of backward
chaining. It initializes the ``ToDo'' set $\mathcal{T}$ to all homomorphic
images of the answers to $\predQ$ on $B'$ (line~\ref{alg:relevance:init}) and
then iteratively explores $\mathcal{T}$
(lines~\ref{alg:relevance:while:start}--\ref{alg:relevance:while:end}). In each
iteration, it extracts a fact $F$ from $\mathcal{T}$
(line~\ref{alg:relevance:choose-F}) and then identifies each rule $r$ and
substitution $\nu$ matching the head of $r$ to $F$ and the body of $r$ in $I$
(line~\ref{alg:relevance:rule:start}). Such $\nu$ captures ways of deriving a
fact represented by $F$ from $B$ via $r$, so $r$ is added to the set
$\mathcal{R}$ of relevant rules if such $\mu$ exists
(line~\ref{alg:relevance:add-r}). Finally, the matched body atoms must be
derivable as well, so they are all added to $\mathcal{T}$
(line~\ref{alg:relevance:add-T-D}). The ``done'' set $\mathcal{D}$ ensures that
each fact is added to $\mathcal{T}$ just once, which ensures termination.

We can optimize the algorithm if ${P \cup B}$ is known to satisfy UNA (e.g., if
an earlier UNA check was conducted). If $\nu$ matches an atom ${x \equals t \in
\body{r}}$ as ${c \equals c}$, the corresponding derivation from $P$ and $B$
necessarily matches ${x \equals t}$ to ${d \equals d}$ for some constant $d$;
due to $\equals$-safety, we can derive ${d \equals d}$ using reflexivity
axioms, so we do not need to examine other proofs for ${d \equals d}$
(line~\ref{alg:relevance:UNA-opt}). Moreover, if all matches of ${x \equals t}$
are of such a form, then we can replace $x$ with $t$ in $r$ and inform the
subsequent magic sets transformation step that no equalities are relevant to
this atom. To this end, Algorithm~\ref{alg:relevance} maintains a set
$\mathcal{B}$ of ``blocked'' body equality atoms that records all body equality
atoms that can be matched to an equality not of the form ${d \equals d}$
(line~\ref{alg:relevance:add-B}). After considering all possibilities for
deriving certain answers, body equality atoms that have not been blocked are
removed (lines~\ref{alg:relevance:UNA:start}--\ref{alg:relevance:UNA:end}).

\begin{algorithm}[tb]
\caption{$\relevance{P}{B}$}\label{alg:relevance}
\begin{small}
\begin{algorithmic}[1]
    \State \textbf{choose} an abstraction $B'$ of $B$ w.r.t.\ $P$                                                           \label{alg:relevance:abstraction}
    \State $I \defeq \fixpoint{P'}{B'}$ for $P' = P \cup \Ref{P} \cup \SymTrans$                                            \label{alg:relevance:abstraction-fixpoint}
    \State $\mathcal{D} \defeq \mathcal{T} \defeq \{ \predQ(\vec a) \in I \mid \vec a \text{ is a tuple of constants} \}$   \label{alg:relevance:init}
    \State $\mathcal{R} \defeq \emptyset$ \;\; and \;\; $\mathcal{B} \defeq \emptyset$
    \While{$\mathcal{T} \neq \emptyset$}                                                                                    \label{alg:relevance:while:start}
        \State \textbf{choose and remove} some fact $F$ from $\mathcal{T}$                                                  \label{alg:relevance:choose-F}
        \For{\textbf{each} $r \in P \cup \SymTrans$ and substitution $\nu$ such that
            \Statex \hspace{1.55cm} $\nu(\head{r}) = F$ and $\nu(\body{r}) \subseteq I$}                                    \label{alg:relevance:rule:start}
            \If{$r \not\in \mathcal{R} \cup \SymTrans$}
                \textbf{add} $r$ to $\mathcal{R}$                                                                           \label{alg:relevance:add-r}
            \EndIf
            \For{\textbf{each} $G_i \in \nu(\body{r})$}
                \If{$G_i$ is not of the form $c \equals c$ with $c$ a constant, or
                    \Statex \hspace{1.6cm} $P \cup B$ is not known to satisfy UNA}                                          \label{alg:relevance:UNA-opt}
                    \If{$G_i \not\in \mathcal{D}$}
                        \textbf{add} $G_i$ to $\mathcal{T}$ and $\mathcal{D}$                                               \label{alg:relevance:add-T-D}
                    \EndIf
                    \If{$G_i$ is an equality}
                        \textbf{add} $\langle r,i \rangle$ to $\mathcal{B}$                                                 \label{alg:relevance:add-B}
                    \EndIf
                \EndIf
            \EndFor
        \EndFor                                                                                                             \label{alg:relevance:rule:end}
    \EndWhile                                                                                                               \label{alg:relevance:while:end}
    \If{$P$ and $B$ and known to satisfy UNA}                                                                               \label{alg:relevance:UNA:start}
        \For{\textbf{each} $r \in \mathcal{R}$ and each $i$-th atom of $\body{r}$ of the form
            \Statex \hspace{1.55cm} $x \equals t$ with $t$ a term and $\langle r,i \rangle \not\in \mathcal{B}$}
            \State \textbf{remove} $x \equals t$ from $r$ and replace $x$ with $t$ in $r$                                   \label{alg:relevance:UNA:desg}
        \EndFor
    \EndIf                                                                                                                  \label{alg:relevance:UNA:end}
    \State \Return $\mathcal{R}$
\end{algorithmic}
\end{small}
\end{algorithm}

The computation of $I$ in line~\ref{alg:relevance:abstraction-fixpoint} may not
terminate in general, but we shall apply Algorithm~\ref{alg:relevance} in our
pipeline only in cases where termination is guaranteed.
Theorem~\ref{theorem:relevance} shows that, in that case, the query answers
remain preserved.

\begin{restatable}{theorem}{thmrelevance}\label{theorem:relevance}
    For each $\equals$-safe program $P$ defining the query predicate $\predQ$
    and base instance $B$ where line~\ref{alg:relevance:abstraction-fixpoint}
    of Algorithm~\ref{alg:relevance} terminates, program ${\mathcal{R} =
    \relevance{P}{B}}$ is $\equals$-safe, and, for each tuple ${\vec a}$ of
    constants, ${P \cup \Ref{P} \cup \SymTrans \cup B \models \predQ(\vec a)}$
    if and only if ${\mathcal{R} \cup \Ref{\mathcal{R}} \cup \SymTrans \cup B
    \models \predQ(\vec a)}$.
\end{restatable}

If computing $\fixpoint{P'}{B'}$ in line~\ref{alg:relevance:abstraction} is
difficult (as is the case in some of our experiments), we can replace each term
${f(\vec x)}$ in $P$ with a fresh constant $c_f$ unique for $f$. This does not
affect the algorithm's correctness, since $\fixpoint{P'}{B}$ can still be
homomorphically embedded into $\fixpoint{P'}{B'}$. Moreover, the resulting
program then does not contain function symbols, and so the computation in
line~\ref{alg:relevance:abstraction-fixpoint} necessarily terminates.
  
\begin{example}
In our running example, UNA holds on $\exSigma$ and $\exB$, and we shall assume
that this is known in advance. Moreover, we shall take $B'$ to be the critical
instance, containing facts $B(\ast)$ and $S(\ast,\ast)$. Computing the least
fixpoint in line~\ref{alg:relevance:abstraction-fixpoint} of
Algorithm~\ref{alg:relevance} produces the following instance $I$.
\begin{align}
    B(\ast)                && S(\ast,\ast)              && R(\ast,f(\ast))        && f(\ast) \equals f(\ast)  \nonumber \\
    T(\ast,g(\ast))        && A(g(\ast)) \nonumber      && \ast \equals \ast      && \ast \equals g(\ast)      \nonumber \\
    g(\ast) \equals \ast   && g(\ast) \equals g(\ast)   && \predQ(\ast)           && \predQ(g(\ast))           \nonumber
\end{align}
Algorithm~\ref{alg:relevance} starts by considering $\predQ(\ast)$. Matching
the fact to the head of \eqref{runex:Q-sk} and evaluating the body in $I$
produces the ground rule instance
\begin{displaymath}
    \predQ(\ast) \leftarrow A(g(\ast)) \wedge \ast \equals g(\ast) \wedge R(\ast,f(\ast)) \wedge \ast \equals \ast;
\end{displaymath}    
thus, rule \eqref{runex:Q-sk} is identified as relevant. UNA is known to hold,
so atom ${\ast \equals \ast}$ is not considered any further, but the algorithm
must consider the remaining body atoms. Matching ${g(\ast) \equals f(\ast)}$ to
the head of \eqref{runex:egd-T-sk} and evaluating the body in $I$ produces the
ground rule instance ${\ast \equals g(\ast) \leftarrow T(\ast,g(\ast))}$, so
rule \eqref{runex:egd-T-sk} is identified as relevant; moreover, matching
${T(\ast,g(\ast))}$ to the head of \eqref{runex:T-sk} produces the ground rule
instance ${T(\ast,g(\ast)) \leftarrow B(\ast)}$, so rule \eqref{runex:T-sk} is
identified as relevant as well. In contrast, matching ${\ast \equals g(\ast)}$
to the head of \eqref{runex:egd-RS-sk} produces query
\begin{displaymath}
    R(x,\ast) \wedge \equals x'' \wedge S(x'',x') \wedge x' \equals x''' \wedge R(x''',g(\ast)),
\end{displaymath}
which has no matches in $I$; thus, rule \eqref{runex:egd-RS-sk} is not added to
$\mathcal{R}$. Next, matching $R(\ast,f(\ast))$ to the head of
\eqref{runex:R-sk} and evaluating the body in $I$ produces the ground rule
instance ${R(\ast,f(\ast)) \leftarrow S(\ast,\ast)}$; thus, rule
\eqref{runex:R-sk} is identified as relevant. Finally, $B(\ast)$ and
$S(\ast,\ast)$ do not match any rule head, so the algorithm terminates. Thus,
the algorithm returns all rules apart from \eqref{runex:egd-RS-sk}. Moreover,
atom ${x \equals x''}$ in \eqref{runex:Q-sk} is matched to ${\ast \equals
f(\ast)}$ so it cannot be removed---that is, this equality is relevant. In
contrast, atom ${x \equals x'}$ in \eqref{runex:Q-sk} is matched \emph{only} to
${\ast \equals \ast}$; since we know that UNA holds, this equality is
irrelevant and it is removed in
lines~\ref{alg:relevance:UNA:start}--\ref{alg:relevance:UNA:end} of
Algorithm~\ref{alg:relevance}. Consequently, the algorithm returns program
$P_3$ consisting of rules \eqref{runex:Q-rel}--\eqref{runex:egd-T-rel}.
\begin{align}
    \predQ(x)   & \leftarrow A(x'') \wedge x \equals x'' \wedge R(x,y)  \label{runex:Q-rel} \\
    R(x,f(x))   & \leftarrow S(x,z)                                     \label{runex:R-rel} \\
    T(x,g(x))   & \leftarrow B(x)                                       \label{runex:T-rel} \\
    A(g(x))     & \leftarrow B(x)                                       \label{runex:A-rel} \\
    x \equals y & \leftarrow T(x,y)                                     \label{runex:egd-T-rel}
\end{align}
\end{example}    

\section{Magic Sets for Existential Rules with $\equals$} \label{sec:magic}

We now present our variant of the magic sets transformation. Our technique also
mimics top-down evaluation, but with a specialized treatment of equality that
takes the symmetry of $\equals$ into account and further prunes the set of
proofs for facts of the form ${t \equals t}$. In particular, singularization
critically relies on reflexivity axioms, and passing these to the magic sets
algorithm would significantly blow up the resulting rule set. The
$\equals$-safety of the rules implies that our magic transformation does not
need to be applied to the reflexivity rules, which results in a much more
efficient magic program.

We follow \citet{DBLP:journals/jlp/BeeriR91} in defining adornments and magic
predicates for predicates other than $\equals$, but, as we discuss shortly, we
optimize these notions for $\equals$. Intuitively, an adornment identifies
which arguments of an atom will be bound by sideways information passing, and a
magic predicate will ``collect'' the passed arguments.

\begin{definition}\label{def:adornment}
    An \emph{adornment} for an $n$-ary predicate $R$ other than $\equals$ is a
    string $\alpha$ of length $n$ over alphabet $\ad{b}$ (``bound'') and
    $\ad{f}$ (``free''), and $\mgc{R}{\alpha}$ is a fresh \emph{magic}
    predicate unique for $R$ and $\alpha$ with arity equal to the number of
    $\ad{b}$-symbols in $\alpha$. An \emph{adornment} for $\equals$ has the
    form $\ad{bb}$, $\ad{bf}$, or $\ad{fb}$, and $\mgc{\equals}{\ad{bb}}$ and
    $\mgc{\equals}{\adEqb}$ are fresh \emph{magic} predicates for $\equals$ of
    arity two and one, respectively. For $\alpha$ an adornment of length $n$
    and ${\vec t}$ an $n$-tuple of terms, ${\vec t^\alpha}$ contains in the
    same relative order each ${t_i \in \vec t}$ for which the $i$-th element of
    $\alpha$ is $\ad{b}$.
\end{definition}

Definition~\ref{def:adornment} takes into account that, if one argument of an
equality atom is bound, the other argument will also be bound due to the
symmetry of $\equals$. Thus, predicate $\mgc{\equals}{\adEqb}$ is used for both
$\ad{bf}$ and $\ad{fb}$, where notation $\adEqb$ stresses that the positions of
$\ad{b}$ and $\ad{f}$ are interchangeable. Moreover, at least one argument of
an equality atom must be bound, so $\equals$ cannot be adorned by $\ad{ff}$.
Definition~\ref{def:SIPS} introduces a sideways information passing strategy
(SIPS), which determines how information is propagated through the rule bodies.
Function $\mathsf{reorder}$ reorders the rule bodies to maximize information
passing, and function $\mathsf{adorn}$ decides which arguments of an atom
should be bound given a set of available bindings.

\begin{definition}\label{def:SIPS}
    A \emph{sideways information passing strategy} consists of the following
    two functions.
    
    For $\varphi$ a conjunction of atoms and $T$ a set of terms,
    $\mathsf{reorder}(\varphi,T)$ returns an ordering ${\langle R_1(\vec t_1),
    \dots, R_n(\vec t_n) \rangle}$ of the conjuncts of $\varphi$ such that, for
    each equality atom $R_i(\vec t_i)$ of the form ${x \equals y}$ or ${x
    \equals s}$ where $s$ is ground, ${z \in \vars{\vec t_i}}$ exists such that
    ${z \in T}$ or ${z \in \vec t_j}$ for some ${j < i}$ with ${R_j \neq
    {\equals}}$.
    
    For ${R(t_1, \dots, t_k)}$ an atom and $V$ a set of variables,
    ${\mathsf{adorn}(R(t_1, \dots, t_k),V)}$ returns an adornment $\alpha$ for
    $R$ such that ${\vars{t_j} \subseteq V}$ if the $j$-th element of $\alpha$
    is $\ad{b}$.
\end{definition}

Algorithm~\ref{alg:magic} implements the magic sets transformation optimized
for $\equals$. It initializes the ``ToDo'' set $\mathcal{T}$ with the magic
predicate for the query (line~\ref{alg:magic:init}) and processes $\mathcal{T}$
iteratively. For each magic predicate $\mgc{R}{\alpha}$ in $\mathcal{T}$
(line~\ref{alg:magic:R:start}), it identifies each rule $r$ that can derive $R$
(line~\ref{alg:magic:r:start}). Adornment ${\alpha = \adEqb}$ is processed as
both $\ad{bf}$ and $\ad{fb}$
(lines~\ref{alg:magic:process:bf}--\ref{alg:magic:process:fb}), and each other
$\alpha$ is processed as is (line~\ref{alg:magic:process}). In all cases, the
algorithm produces the \emph{modified rule} by restricting the body of $r$ by
the magic predicate corresponding to the head predicate
(line~\ref{alg:magic:mod-rule}), and then it reorders
(line~\ref{alg:magic:reorder}) and processes
(lines~\ref{alg:magic:body:start}--\ref{alg:magic:body:end}) the body of $r$.
For each body atom ${R_i(\vec t_i)}$ with $R_i$ occurring in the head of $P$,
the algorithm uses the SIPS to determine an adornment $\gamma$ identifying the
bound arguments (line~\ref{alg:magic:adorn}), and it generates the \emph{magic
rule} that populates $\mgc{\equals}{\adEqb}$ or $\mgc{R_i}{\gamma}$ with the
bindings for $R_i$ (line~\ref{alg:magic:magic-rule}). The algorithm takes into
account that $\mgc{\equals}{\adEqb}$ captures both $\ad{bf}$ and $\ad{fb}$
(line~\ref{alg:magic:S:start}--\ref{alg:magic:S:end}). The magic rule would not
be $\equals$-safe if ${R_i(\vec t_i)}$ were an equality atom with no arguments
bound, which, at the end of our pipeline, could produce a rule with variables
occurring the head but not the body. Thus, Definition~\ref{def:adornment} does
not introduce the $\ad{ff}$ adornment for $\equals$, and
Definition~\ref{def:SIPS} requires at least one argument of each equality atom
to be bound. Finally, the magic predicate must also be processed
(line~\ref{alg:magic:add-S}), and the ``done'' set $\mathcal{D}$ ensures that
this happens just once.

\begin{algorithm}[tb]
\caption{$\magic{P}$}\label{alg:magic}
\begin{small}
\begin{algorithmic}[1]
    \State $\mathcal{D} \defeq \mathcal{T} \defeq \{ \mgc{\predQ}{\alpha} \}$ and $\mathcal{R} \defeq \{ \mgc{\predQ}{\alpha} \leftarrow \}$, \;\; $\alpha = \ad{f} \cdots \ad{f}$  \label{alg:magic:init}
    \While{$\mathcal{T} \neq \emptyset$}
        \State \textbf{choose and remove} some $\mgc{R}{\alpha}$ from $\mathcal{T}$                                                                                                 \label{alg:magic:R:start}
        \For{\textbf{each} ${r \in P \cup \SymTrans}$ such that $\head{r} = R(\vec t)$}                                                                                             \label{alg:magic:r:start}
            \If{$R = {\equals}$ and $\alpha = \adEqb$}
                \State $\mathsf{process}(r,\alpha,\ad{bf})$                                                                                                                         \label{alg:magic:process:bf}
                \State $\mathsf{process}(r,\alpha,\ad{fb})$                                                                                                                         \label{alg:magic:process:fb}
            \Else
                \State $\mathsf{process}(r,\alpha,\alpha)$                                                                                                                          \label{alg:magic:process}
            \EndIf
        \EndFor                                                                                                                                                                     \label{alg:magic:r:end}
    \EndWhile                                                                                                                                                                       \label{alg:magic:R:end}
    \State \Return $\mathcal{R}$                                                                                                                                                    \label{alg:magic:return}
    \Statex
    \Procedure{$\mathsf{process}$}{$r, \alpha, \beta$} where $\head{r} = R(\vec t)$
        \State \textbf{add} $\head{r} \leftarrow \mgc{R}{\alpha}(\vec t^\beta) \wedge \body{r}$ to $\mathcal{R}$                                                                    \label{alg:magic:mod-rule}
        \State $\langle R_1(\vec t_1), \dots, R_n(\vec t_n) \rangle \defeq \mathsf{reorder}(\body{r},\vec t^\beta)$                                                                 \label{alg:magic:reorder}
        \For{$1 \leq i \leq n$ s.t.\ $R_i$ is $\equals$ or it occurs in $P$ in a head}                                                                                              \label{alg:magic:body:start}
            \State $\gamma \defeq \mathsf{adorn}(R_i(\vec t_i),\vars{\vec t^\beta} \cup \bigcup_{j=1}^{i-1} \vars{\vec t_j})$                                                       \label{alg:magic:adorn}
            \If{$R_i = {\equals}$ and $\gamma \in \{ \ad{bf}, \ad{fb} \}$}                                                                                                          \label{alg:magic:S:start}
                $S \defeq \mgc{\equals}{\adEqb}$
            \Else\,
                $S \defeq \mgc{R_i}{\gamma}$
            \EndIf                                                                                                                                                                  \label{alg:magic:S:end}
            \State \textbf{add} $S(\vec t_i^\gamma) \leftarrow \mgc{R}{\alpha}(\vec t^\beta) \wedge \bigwedge_{j=1}^{i-1} R_j(\vec t_j)$ to $\mathcal{R}$                           \label{alg:magic:magic-rule}
            \If{$S \not \in \mathcal{D}$}
                \textbf{add} $S$ to $\mathcal{T}$ and $\mathcal{D}$                                                                                                                 \label{alg:magic:add-S}
            \EndIf
        \EndFor                                                                                                                                                                     \label{alg:magic:body:end}
    \EndProcedure
\end{algorithmic}
\end{small}
\hrule
\textbf{Note:} please remember that $t_1 \equals t_2$ can also be written as
${\equals}(t_1,t_2)$, so $R(\vec t)$ and $R_i(\vec t_i)$ can be equality atoms.
\end{algorithm}

Theorem~\ref{theorem:magic} shows that Algorithm~\ref{alg:magic} preserves
$\equals$-safety and query answers, and Theorem~\ref{theorem:termination} shows
that it also preserves chase termination. The latter does not hold for all
programs with function symbols; for example, transforming ${A(x) \leftarrow
A(f(x))}$ can produce the nonterminating rule ${\mgc{A}{f}(f(x)) \leftarrow
\mgc{A}{f}(x)}$. However, in Algorithm~\ref{alg:answer-query}, the magic sets
are applied in line~\ref{alg:answer-query:magic} only to programs $P_3$ that do
not contain function symbols in the body, which suffices to show that the heads
of the magic rules are function-free and cannot derive terms of unbounded
depth, and therefore the transformation does not affect termination.

\begin{restatable}{theorem}{thmmagic}\label{theorem:magic}
    For each $\equals$-safe program $P$ defining the query predicate $\predQ$,
    program ${\mathcal{R} = \magic{P}}$ is $\equals$-safe; moreover, for each
    base instance $B$ and each tuple ${\vec t}$ of ground terms,
    \begin{displaymath}
    \begin{array}{@{}r@{\;}l@{\;\;}l@{}}
        P \cup \Ref{P} \cup \SymTrans \cup B                        & \models \predQ(\vec t)    & \text{iff} \\
        \mathcal{R} \cup \Ref{\mathcal{R}} \cup \SymTrans \cup B    & \models \predQ(\vec t).
    \end{array}
    \end{displaymath}
\end{restatable}

\begin{restatable}{theorem}{thmtermination}\label{theorem:termination}
    Let $P$ be a program where the body of each rule is function-free, and let
    ${P_1 = P \cup \Ref{P} \cup \SymTrans}$ and ${P_2 = \mathcal{R} \cup
    \Ref{\mathcal{R}} \cup \SymTrans}$ for ${\mathcal{R} = \magic{P}}$. For
    each base instance $B$, if ${\fixpoint{P_1}{B}}$ is finite, then
    ${\fixpoint{P_2}{B}}$ is finite as well.
\end{restatable}

\begin{example}
Applying Algorithm~\ref{alg:magic} to $P_3$ produces program $P_4$ consisting
of rules \eqref{ex:mgc:start}--\eqref{ex:mgc:end}. Horizontal lines separate
the rules produced in each invocation of $\mathsf{process}$. Please note that
\eqref{runex:egd-T-rel} produces \eqref{ex:mgc1-1} and \eqref{ex:mgc1-2} when
$\adEqb$ is interpreted as $\ad{bf}$, and \eqref{ex:mgc1-1} and
\eqref{ex:mgc1-2} when $\adEqb$ is interpreted as $\ad{bf}$.
\begin{align}
    \mgc{\predQ}{\ad{f}}        & \leftarrow                                                                        \label{ex:mgc:start} \\
    \hline
    \predQ(x)                   & \leftarrow \mgc{\predQ}{\ad{f}} \wedge A(x'') \wedge x \equals x'' \wedge R(x,y) \\
    \mgc{A}{\ad{f}}             & \leftarrow \mgc{\predQ}{\ad{f}} \\
    \mgc{\equals}{\adEqb}(x'')  & \leftarrow \mgc{\predQ}{\ad{f}} \wedge A(x'') \\
    \mgc{R}{\ad{bf}}(x)         & \leftarrow \mgc{\predQ}{\ad{f}} \wedge A(x'') \wedge x \equals x'' \\
    \hline
    A(g(x))                     & \leftarrow \mgc{A}{\ad{f}} \wedge B(x) \\
    \hline
    x \equals y                 & \leftarrow \mgc{\equals}{\adEqb}(x) \wedge T(x,y)                                 \label{ex:mgc1-1} \\
    \mgc{T}{\ad{bf}}(x)         & \leftarrow \mgc{\equals}{\adEqb}(x)                                               \label{ex:mgc1-2} \\
    \hline
    x \equals y                 & \leftarrow \mgc{\equals}{\adEqb}(y) \wedge T(x,y)                                 \label{ex:mgc2-1} \\
    \mgc{T}{\ad{fb}}(y)         & \leftarrow \mgc{\equals}{\adEqb}(y)                                               \label{ex:mgc2-2} \\
    \hline
    T(x,g(x))                   & \leftarrow \mgc{T}{\ad{bf}}(x) \wedge B(x)                                        \label{ex:mgcT-1} \\
    \hline
    T(x,g(x))                   & \leftarrow \mgc{T}{\ad{fb}}(g(x)) \wedge B(x)                                     \label{ex:mgcT-2} \\
    \hline
    R(x,f(x))                   & \leftarrow \mgc{R}{\ad{bf}}(x) \wedge S(x,z)                                      \label{ex:mgc:end}
\end{align}
\end{example}

\section{Final Transformations}\label{sec:final}

The final steps ensure that the resulting program can be evaluated efficiently
using the chase for logic programs, which, as explained in
Section~\ref{sec:preliminaries}, can handle only programs with no constants,
function symbols, and $\equals$ in the rule bodies. The magic sets
transformation can introduce body atoms with function symbols, so
Definition~\ref{def:defun} removes these by introducing fresh predicates.
Proposition~\ref{prop:defun} shows the query answers remain preserved since the
fresh predicates can always be interpreted to reflect the structure of the
ground functional terms encountered during the chase.

\begin{definition}\label{def:defun}
    Program $\defun{P}$ is obtained from a program $P$ by exhaustively applying
    the following steps.
    \begin{enumerate}
        \item In the body of each rule, replace each occurrence of a constant
        $c$ with a fresh variable $z_c$ unique for $c$, add atom
        ${\funpred{c}(z_t)}$ to the body, and add the rule ${\funpred{c}(c)
        \leftarrow}$.
        
        \item In the body of each rule, replace each occurrence of a term $t$
        of the form ${f(\vec s)}$ with a fresh variable $z_t$ unique for $t$,
        and add atom ${\funpred{f}(\vec s,z_t)}$ to the body.
        
        \item For each rule $r$ and each term of the form ${f(\vec s)}$
        occurring in $\head{r}$ with $f$ a function symbol considered in the
        second step, add the rule ${\funpred{f}(\vec s, f(\vec s)) \leftarrow
        \body{r}}$.
    \end{enumerate}
\end{definition}

\begin{restatable}{proposition}{propdefun}\label{prop:defun}
    For each program $P$ and ${P' = \defun{P}}$, base instance $B$, predicate
    $R$ not of the form $\funpred{f}$, and tuple ${\vec t}$ of ground terms,
    ${P \cup \Ref{P} \cup \SymTrans \cup B \models R(\vec t)}$ if and only if
    ${P' \cup \Ref{P'} \cup \SymTrans \cup B \models R(\vec t)}$.
\end{restatable}

Definition~\ref{def:desg} reverses the effects of singularization and removes
all body equality atoms. As a side-effect, this reduces the number of rule
variables, which simplifies rule matching.

\begin{definition}\label{def:desg}
    The \emph{desingularization} of a rule is obtained by repeatedly removing
    each body atom of the form ${x \equals t}$ while replacing $x$ with $t$
    everywhere in the rule. For $P$ a program, $\desg{P}$ contains a
    desingularization of each rule of $P$.
\end{definition}

We evaluate the final program using the chase for logic programs, which
captures the effects of congruence axioms. Note, however, that program $P_5$
from Algorithm~\ref{alg:answer-query} contains fresh predicates introduced by
the magic sets transformation and the elimination of function symbols, to which
the chase will (implicitly) apply congruence axioms as well.
Theorem~\ref{theorem:desg} shows that this preserves the query answers, and its
proof is not trivial: adding congruence axioms to a program produces new
consequences, so the proof depends on the fact program $P_5$ was obtained as
shown in Algorithm~\ref{alg:answer-query}.

\begin{restatable}{theorem}{thmdesg}\label{theorem:desg}
    For each finite set of existential rules $\Sigma$ defining the query
    predicate $\predQ$, each base instance $B$, each tuple of constants ${\vec
    a}$, and program $P_6$ obtained from $\Sigma$ and $B$ by applying
    Algorithm~\ref{alg:answer-query}, ${\Sigma \cup B \modelsEq \predQ(\vec
    a)}$ if and only if ${P_6 \cup \Ref{P_6} \cup \Cong{P_6} \cup B \models
    \predQ(\vec a)}$.
\end{restatable}

Theorem~\ref{thm:query-answering} shows that our entire pipeline is correct.
Note that, if the Skolem chase of $\sg{\Sigma}$ terminates on every base
instance, then line~\ref{alg:relevance:abstraction-fixpoint} of
Algorithm~\ref{alg:relevance} necessarily terminates.

\begin{restatable}{theorem}{thmalgcorrect}\label{thm:query-answering}
    For each finite set of existential rules $\Sigma$ defining the query
    predicate $\predQ$ such that the chase of $\sg{\Sigma}$ terminates on all
    base instances, and for each base instance $B$,
    Algorithm~\ref{alg:answer-query} outputs precisely all answers to $\predQ$
    on ${\Sigma \cup B}$ and then terminates.
\end{restatable}

\begin{example}\label{ex:final}
In our running example, program $P_6$ contains rules
\eqref{ex:fin:start}--\eqref{ex:fin:R}, where \eqref{ex:mgcT-1} produces
\eqref{ex:fin:T-1} and \eqref{ex:fin:T-1-f}, and \eqref{ex:mgcT-2} produces
\eqref{ex:fin:T-2} and \eqref{ex:fin:T-2-f}.
\begin{align}
    \mgc{\predQ}{\ad{f}}        & \leftarrow                                                                        \label{ex:fin:start} \\
    \hline
    \predQ(x)                   & \leftarrow \mgc{\predQ}{\ad{f}} \wedge A(x) \wedge R(x,y)                         \label{ex:fin:Q} \\
    \mgc{A}{\ad{f}}             & \leftarrow \mgc{\predQ}{\ad{f}} \\
    \mgc{\equals}{\adEqb}(x'')  & \leftarrow \mgc{\predQ}{\ad{f}} \wedge A(x'') \\
    \mgc{R}{\ad{bf}}(x'')       & \leftarrow \mgc{\predQ}{\ad{f}} \wedge A(x'') \\
    \hline
    A(g(x))                     & \leftarrow \mgc{A}{\ad{f}} \wedge B(x) \\
    \funpred{g}(x,g(x))         & \leftarrow \mgc{A}{\ad{f}} \wedge B(x) \\
    \hline
    x \equals y                 & \leftarrow \mgc{\equals}{\adEqb}(x) \wedge T(x,y)                                 \label{ex:fin:egd-T-1} \\
    \mgc{T}{\ad{bf}}(x)         & \leftarrow \mgc{\equals}{\adEqb}(x) \\
    \hline
    x \equals y                 & \leftarrow \mgc{\equals}{\adEqb}(y) \wedge T(x,y)                                 \label{ex:fin:egd-T-2} \\
    \mgc{T}{\ad{fb}}(y)         & \leftarrow \mgc{\equals}{\adEqb}(y) \\
    \hline
    T(x,g(x))                   & \leftarrow \mgc{T}{\ad{bf}}(x) \wedge B(x)                                        \label{ex:fin:T-1} \\
    \funpred{g}(x,g(x))         & \leftarrow \mgc{T}{\ad{bf}}(x) \wedge B(x)                                        \label{ex:fin:T-1-f} \\
    \hline
    T(x,g(x))                   & \leftarrow \funpred{g}(x,z_{g(x)}) \wedge \mgc{T}{\ad{fb}}(z_{g(x)}) \wedge B(x)  \label{ex:fin:T-2} \\
    \funpred{g}(x,g(x))         & \leftarrow \funpred{g}(x,z_{g(x)}) \wedge \mgc{T}{\ad{fb}}(z_{g(x)}) \wedge B(x)  \label{ex:fin:T-2-f} \\
    \hline
    R(x,f(x))                   & \leftarrow \mgc{R}{\ad{bf}}(x) \wedge S(x,z)                                      \label{ex:fin:R}
\end{align}

Program $P_6$ contains no constants, function symbols, or equality atoms in the
body, so we can answer $\predQ$ on $P_6$ and $\exB$ using the chase for logic
programs and thus avoid computing the least fixpoint for a program that
explicitly axiomatizes equality. Doing so derives the following facts.
\begin{displaymath}
\begin{array}{@{}llll@{}}
    \mgc{\predQ}{\ad{f}}            & \mgc{A}{\ad{f}}           & A(g(a_1))     & \funpred{g}(a_1,g(a_1)) \\
    \mgc{\equals}{\adEqb}(g(a_1))   & \mgc{T}{\ad{fb}}(g(a_1))  & T(a_1,g(a_1)) \\
\end{array}
\end{displaymath}
Next, rule \eqref{ex:fin:egd-T-2} derives ${a_1 \equals g(a_1)}$, so the chase
for logic programs takes $a_1$ as the representative of $g(a_1)$ and replaces
$g(a_1)$ with $a_1$, thus deriving the following facts.
\begin{displaymath}
\begin{array}{@{}llll@{}}
    \mgc{\predQ}{\ad{f}}        & \mgc{A}{\ad{f}}       & A(a_1)        & \funpred{g}(a_1,a_1) \\
    \mgc{\equals}{\adEqb}(a_1)  & \mgc{T}{\ad{fb}}(a_1) & T(a_1,a_1)
\end{array}
\end{displaymath}
After this, the chase further derives the following facts.
\begin{displaymath}
     \mgc{T}{\ad{bf}}(a_1) \qquad \mgc{R}{\ad{bf}}(a_1) \qquad R(x,f(a_1)) \qquad \predQ(a_1)
\end{displaymath}
Note that no facts involving $a_i$ with ${i \geq 2}$ are derived, as these are
irrelevant to answering $\predQ$. Thus, evaluating $P_6$ on $\exB$ produces far
fewer facts than applying the chase for logic programs to $\sk{\exSigma}$.
\end{example}

\section{Empirical Evaluation}\label{sec:evaluation}

We evaluated our technique using $\chaseBench$ \cite{chasebench}, a recent
benchmark offering a mix of scenarios that simulate data exchange and ontology
reasoning applications. We selected the scenarios summarized in
Table~\ref{table:test-scenarios}, each comprising a set of existential rules, a
base instance, and several queries. $\LUBMonehun$ and $\LUBMonek$ are derived
from the well-known Semantic Web LUBM~\cite{lubm} benchmark; $\Deepthreehun$ is
a ``stress test'' scenario; $\Doctorsonem$ simulates data exchange between
medical databases; and $\STB$ and $\Ont$ were produced using the $\ibench$ and
$\toxgene$ rule and instance generators. The first three scenarios contain only
TGDs, and the remaining ones contain EGDs as well. All rules are weakly
acyclic, so the chase always terminates. Finally, UNA was known to hold in all
cases \cite{chasebench}.

To compute the chase of the final program (line~\ref{alg:answer-query:chase} of
Algorithm~\ref{alg:answer-query}), we used the RAM-based RDFox system written
in C++.\footnote{\url{http://www.cs.ox.ac.uk/isg/tools/RDFox/}} We implemented
our technique in Java on top of the $\chaseBench$ \cite{chasebench} library. We
used just one thread while computing the chase. Our system and the test data
are available online.\footnote{\url{http://github.com/tsamoura/chaseGoal}}

\begin{table}[tb]
\centering
\begin{footnotesize}
\begin{tabular}{l@{\;}|@{\;}r@{\;}|@{\;}r@{\;}|@{\;}r@{\;}|@{\;}r@{\;}|@{\;}r}
                    & \multirow{2}{*}{TGDs} & \multirow{2}{*}{EGDs} & \multirow{2}{*}{Facts}    & \multicolumn{2}{@{}c@{}}{Queries} \\
                    &                       &                       &                           & free  & const. \\
    \hline
    \LUBMonehun     & $136$                 & $0$                   & \SI{12}{\mega\nothing}    & $4$   & $10$ \\
    \LUBMonek       & $136$                 & $0$                   & \SI{120}{\mega\nothing}   & $4$   & $10$ \\
    \Deepthreehun   & $1,300$               & $0$                   & \SI{1}{\kilo\nothing}     & $20$  & $0$ \\
    \Doctorsonem    & $5$                   & $4$                   & \SI{1}{\mega\nothing}     & $7$   & $11$ \\
    \STB            & $199$                 & $93$                  & \SI{1}{\mega\nothing}     & $24$  & $26$ \\
    \Ont            & $529$                 & $348$                 & \SI{2}{\mega\nothing}     & $34$  & $8$ \\
    \hline
    \multicolumn{6}{@{}l@{}}{} \\[-1ex]
    \multicolumn{6}{@{}l@{}}{\textbf{Note:} ``const.'' and ``free'' are the numbers of} \\
    \multicolumn{6}{@{}l@{}}{queries with and without constants, respectively.} \\
\end{tabular}
\end{footnotesize}
\caption{Summary of the test scenarios}\label{table:test-scenarios}
\end{table}

No existing goal-driven query answering techniques can handle these scenarios,
as explained in the introduction. Thus, we primarily compared the performance
of computing the chase with no optimizations ($\materialization$), with the
pipeline that uses just the magic sets and omits the relevance analysis
($\magicrun$), with the pipeline that uses just the relevance analysis and
omits the magic sets ($\relevancerun$), and the entire pipeline with both
relevance analysis and magic sets ($\allrun$). The $\materialization$ variant
thus provides us with a baseline, and the remaining tests allow us to identify
the relative contribution of various steps of our pipeline. Note that, in the
presence of EGDs, we always used singularization and the other relevant steps
of our pipeline, rather than the congruence axioms. On TGDs only, our algorithm
becomes equivalent to the classical algorithm of
\citet{DBLP:journals/jlp/BeeriR91}.

In each test run, we computed the program $P_6$ from
Algorithm~\ref{alg:answer-query} (skipping the relevance analysis and/or magic
sets, as required for the test type), computed $\chase{P_6}{B}$, and output the
certain answers of $\predQ$ as shown in line \ref{alg:answer-query:predQ:end}
of Algorithm~\ref{alg:answer-query}. We recorded the wall-clock time of each
run (without the loading times) and the number of facts derived by the chase;
the latter provides an implementation-independent measure of the work needed to
answer a query. In line~\ref{alg:relevance:abstraction} of
Algorithm~\ref{alg:relevance}, we abstracted the base instance using the typed
critical instance; however, computing the least fixpoint of such an abstraction
was infeasible on $\Deepthreehun$ so, in this case only, we used the
optimization from Section~\ref{sec:relevance}.

\begin{figure*}[tb]\centering
\input{figures/results-box-plot}
\captionof{figure}{Times and the numbers of derived facts. Black, green, blue
and red are $\materialization$, $\relevancerun$, $\magicrun$ and $\allrun$,
respectively.}\label{figure:results}
\vspace{0.5cm}
\begin{tabular}{c|c|c|c|c|c}
            & $\LUBM$                   & $\Deepthreehun$           & $\Doctorsonem$            & $\STB$                    & $\Ont$ \\
    \hline
    Time    & \SI{10}{\milli\second}    & \SI{4700}{\milli\second}  & \SI{10}{\milli\second}    & \SI{200}{\milli\second}   & \SI{650}{\milli\second}    \\
\end{tabular}
\captionof{table}{Times for computing the least fixpoints of the base instance abstractions}\label{table:abstraction-chase-time}
\end{figure*}

Figure~\ref{figure:results} summarizes the query times and the numbers of
derived facts for our 158 test queries. The whiskers of each box plot show the
minimum and maximum values, the box shows the lower quartile, the median, and
the upper quartile, and the diamond shows the average. The distributions of
these values are shown in more detail in
\iftoggle{withappendix}{Appendix~\ref{sec:full_results}}{the appendix of the
extended version \cite{extended-version}}.
Table~\ref{table:abstraction-chase-time} shows the times for computing the
least fixpoint of the abstraction in
line~\ref{alg:relevance:abstraction-fixpoint} of Algorithm~\ref{alg:relevance},
which are insignificant in all cases apart from $\Deepthreehun$. On
$\relevancerun$ and $\allrun$, one query of $\Deepthreehun$ and nine queries of
$\STB$ could not be processed by the relevance analysis for reasons we discuss
shortly. Moreover, all queries of $\LUBMonek$ and $\Deepthreehun$ on
$\materialization$, three queries of $\LUBMonek$ and one of $\Deepthreehun$ on
$\magicrun$, all queries of $\LUBMonek$ and one of $\Deepthreehun$ on
$\relevancerun$, and three queries of $\LUBMonek$ on $\allrun$ could not be
processed due to memory exhaustion while computing the chase.

Figure \ref{figure:results} clearly shows that our technique is generally very
effective and can mean the difference between success and failure: the chase
for $\LUBMonek$ and $\Deepthreehun$ could not be computed on our test machine,
whereas $\allrun$ can answer 11 out of 14 queries on $\LUBMonek$ in at most
\SI{45}{\second}, and 19 out of 20 queries on $\Deepthreehun$ in at most
\SI{18}{\second}. The upper quartile of query times for $\allrun$ is at least
an order of magnitude below the times for $\materialization$ in all cases apart
from $\Doctorsonem$, where this holds for the median. Overall, $\allrun$
achieves the best performance.

In addition, relevance analysis alone can lead to significant improvements: the
query times for $\relevancerun$ and $\allrun$ are almost identical on
$\Deepthreehun$ and $\Ont$, suggesting that the improvements are due to
relevance analysis, rather than magic sets. On some queries of $\STB$ and
$\Ont$, the relevance analysis eliminates all rules and thus proves that
queries have no answers. The benefits of relevance analysis are marginal only
on $\LUBM$, mainly because its TGDs contain few existential quantifiers.

However, relevance analysis also has its pitfalls: we could not run it on one
query of $\Deepthreehun$ with eight body atoms, and on nine queries of $\STB$
containing between 11 and 19 output variables that, after singularization, have
between 19 and 22 body atoms. Atoms of these queries match to many facts in the
least fixpoint of the abstraction, so query evaluation explodes either in
line~\ref{alg:relevance:abstraction-fixpoint} or
line~\ref{alg:relevance:rule:start} of Algorithm~\ref{alg:relevance}. One
additional query of $\Deepthreehun$ exhibited similar issues, which we
addressed by (manually) tree-decomposing the query and thus reducing the number
of matches.

\begin{table*}[tb]
\centering\footnotesize
\newcommand{\scen}[1]{\multirow{2}{*}{\ensuremath{#1}}}
\newcommand{\val}[2]{#1 & #2}
\newcommand{\NA}{\multicolumn{2}{@{\;}c@{\;}|@{\;}}{N/A}}
\newcommand{\LNA}{\multicolumn{2}{@{\;}c@{\;}||@{\;}}{N/A}}
\newcommand{\hmin}{\multicolumn{2}{@{\;}c@{\;}|@{\;}}{min}}
\newcommand{\hmax}{\multicolumn{2}{@{\;}c@{\;}|@{\;}}{max}}
\newcommand{\hmed}{\multicolumn{2}{@{\;}c@{\;}||@{\;}}{median}}
\newcommand{\hlmed}{\multicolumn{2}{@{\;}c@{\;}}{median}}
\begin{tabular}{c|c|r@{\,}l@{\;}|@{\;}r@{\,}l@{\;}|@{\;}r@{\,}l@{\;}||@{\;}r@{\,}l@{\;}|@{\;}r@{\,}l@{\;}|@{\;}r@{\,}l@{\;}||@{\;}r@{\,}l@{\;}|@{\;}r@{\,}l@{\;}|@{\;}r@{\,}l@{\;}}
    \hline
    \multicolumn{2}{c|}{}       & \\[-1ex]
    \multicolumn{2}{c|}{}       & \multicolumn{18}{@{\;}c@{\;}}{\textbf{Query times (seconds)}} \\[1ex]
    \cline{3-20}
    \multicolumn{2}{c|}{}       & \multicolumn{6}{@{\;}c@{\;}||@{\;}}{$\magicrun$}          & \multicolumn{6}{@{\;}c@{\;}||@{\;}}{$\relevancerun$}      & \multicolumn{6}{@{\;}c@{\;}}{$\allrun$}                \\
    \cline{3-20}
    \multicolumn{2}{c|}{}       & \hmin             & \hmax             & \hmed             & \hmin             & \hmax             & \hmed             & \hmin             & \hmax             & \hlmed         \\
    \hline
    \scen{\LUBMonehun}      & N & \val{1.62}{}      & \val{93.56}{}     & \val{11.43}{}     & \val{1.69}{}      & \val{26.24}{}     & \val{25.29}{}     & \val{1.59}{}      & \val{77.14}{}     & \val{11.30}{}  \\
                            & Y & \val{0.86}{}      & \val{2.98}{}      & \val{2.58}{}      & \val{0.95}{}      & \val{26.29}{}     & \val{24.92}{}     & \val{0.85}{}      & \val{3.22}{}      & \val{2.68}{}   \\
    \hline
    \scen{\LUBMonek}        & N & \val{19.74}{}     & \val{19.74}{}     & \val{19.74}{}     & \NA               & \NA               & \LNA              & \val{19.74}{}     & \val{19.74}{}     & \val{19.74}{}  \\
                            & Y & \val{9.64}{}      & \val{40.10}{}     & \val{36.19}{}     & \NA               & \NA               & \LNA              & \val{8.66}{}      & \val{44.93}{}     & \val{35.79}{}  \\
    \hline
    \scen{\Doctorsonem}     & N & \val{47.27}{}     & \val{69.21}{}     & \val{54.86}{}     & \val{14.05}{}     & \val{22.21}{}     & \val{16.11}{}     & \val{25.57}{}     & \val{37.81}{}     & \val{29.81}{}  \\
                            & Y & \val{4.28}{}      & \val{8.51}{}      & \val{5.66}{}      & \val{13.75}{}     & \val{15.57}{}     & \val{14.27}{}     & \val{1.15}{}      & \val{2.70}{}      & \val{2.10}{}   \\
    \hline
    \scen{\STB}             & N & \val{14.65}{}     & \val{38.35}{}     & \val{30.07}{}     & \val{0.45}{}      & \val{150.57}{}    & \val{0.74}{}      & \val{0.46}{}      & \val{150.21}{}    & \val{0.75}{}   \\
                            & Y & \val{12.87}{}     & \val{17.92}{}     & \val{14.35}{}     & \val{0.75}{}      & \val{18.47}{}     & \val{1.10}{}      & \val{0.51}{}      & \val{18.18}{}     & \val{0.59}{}   \\
    \hline
    \scen{\Ont}             & N & \val{60.36}{}     & \val{1.49}{k}     & \val{745.17}{}    & \val{1.32}{}      & \val{8.91}{}      & \val{1.86}{}      & \val{1.32}{}      & \val{9.04}{}      & \val{1.99}{}   \\
                            & Y & \val{37.27}{}     & \val{703.26}{}    & \val{690.82}{}    & \val{1.69}{}      & \val{2.69}{}      & \val{1.83}{}      & \val{1.51}{}      & \val{2.28}{}      & \val{1.55}{}   \\
    \hline
    \hline
    \multicolumn{2}{c|}{}       & \\[-1ex]
    \multicolumn{2}{c|}{}       & \multicolumn{18}{@{\;}c@{\;}}{\textbf{The numbers of derived facts}} \\[1ex]
    \cline{3-20}
    \multicolumn{2}{c|}{}       & \multicolumn{6}{@{\;}c@{\;}||@{\;}}{$\magicrun$}          & \multicolumn{6}{@{\;}c@{\;}||@{\;}}{$\relevancerun$}      & \multicolumn{6}{@{\;}c@{\;}}{$\allrun$}                \\
    \cline{3-20}
    \multicolumn{2}{c|}{}       & \hmin             & \hmax             & \hmed             & \hmin             & \hmax             & \hmed             & \hmin             & \hmax             & \hlmed         \\
    \hline
    \scen{\LUBMonehun}      & N & \val{1.59}{M}     & \val{59.23}{M}    & \val{5.58}{M}     & \val{1.59}{M}     & \val{18.42}{M}    & \val{1.74}{M}     & \val{1.59}{M}     & \val{52.41}{M}    & \val{5.33}{M}  \\
                            & Y & \val{18}{}        & \val{51.94}{k}    & \val{4.70}{k}     & \val{92.34}{k}    & \val{19.52}{M}    & \val{17.37}{M}    & \val{18.00}{}     & \val{46.40}{k}    & \val{4.40}{k}  \\
    \hline
    \scen{\LUBMonek}        & N & \val{15.84}{M}    & \val{15.84}{M}    & \val{15.84}{M}    & \NA               & \NA               & \LNA              & \val{15.84}{M}    & \val{15.84}{M}    & \val{15.84}{M} \\
                            & Y & \val{18}{}        & \val{2.77}{M}     & \val{4.40}{k}     & \NA               & \NA               & \LNA              & \val{18.00}{}     & \val{2.77}{M}     & \val{4.40}{k}  \\
    \hline
    \scen{\Doctorsonem}     & N & \val{6.65}{M}     & \val{9.82}{M}     & \val{7.44}{M}     & \val{953.50}{k}   & \val{1.74}{M}     & \val{1.74}{M}     & \val{1.90}{M}     & \val{2.69}{M}     & \val{2.69}{M}  \\
                            & Y & \val{5.94}{k}     & \val{264.30}{k}   & \val{263.16}{k}   & \val{952.50}{k}   & \val{953.63}{k}   & \val{952.50}{k}   & \val{15.00}{}     & \val{4.08}{k}     & \val{2.96}{k}  \\
    \hline
    \scen{\STB}             & N & \val{1.88}{M}     & \val{10.80}{M}    & \val{7.76}{M}     & \val{0.00}{}      & \val{90.05}{k}    & \val{39.47}{k}    & \val{1.00}{}      & \val{90.64}{k}    & \val{40.00}{k} \\
                            & Y & \val{723.19}{k}   & \val{2.92}{M}     & \val{1.05}{M}     & \val{30.00}{k}    & \val{90.00}{k}    & \val{65.00}{k}    & \val{9.00}{}      & \val{51.00}{}     & \val{20.00}{}  \\
    \hline
    \scen{\Ont}             & N & \val{6.35}{M}     & \val{283.68}{M}   & \val{254.05}{M}   & \val{0.00}{}      & \val{118.09}{k}   & \val{59.00}{k}    & \val{1.00}{}      & \val{147.49}{k}   & \val{45.91}{k} \\
                            & Y & \val{1.59}{M}     & \val{224.23}{M}   & \val{224.22}{M}   & \val{29.36}{k}    & \val{118.07}{k}   & \val{54.12}{k}    & \val{9.00}{}      & \val{41.00}{}     & \val{11.50}{}  \\
    \hline
\end{tabular}
\caption{Running times and the numbers of facts for queries without (``N'') and with (``Y'') constants}\label{table:impact-of-constants}
\end{table*}

To investigate the cases in which magic sets are particularly beneficial,
Table~\ref{table:impact-of-constants} shows the minimum, maximum, and median of
the query times and the numbers of derived facts for queries without and with
constants. The maximum numbers of derived facts are particularly telling:
without constants, $\allrun$ derives more facts compared to $\relevancerun$;
and with constants, the numbers for $\allrun$ are several orders of magnitude
smaller compared to $\relevancerun$. Programs $P_3$ in
line~\ref{alg:answer-query:relevance} of Algorithm~\ref{alg:answer-query} are
the same in both cases so this improvement is clearly due to magic sets. In
contrast, the impact of constants is insignificant for $\relevancerun$,
highlighting the different strengths of relevance analysis and magic sets.

Unfortunately, magic sets are not ``free''. For example, $\materialization$ and
$\relevancerun$ are faster than both $\magicrun$ and $\allrun$ on seven
constant-free queries of $\Doctorsonem$, and they outperform $\magicrun$ on 17
queries of $\STB$ and on 39 queries of $\Ont$. In all these cases, the magic
sets transformation increases the number of rules by one or two orders of
magnitude, which introduces considerable overhead during chase computation.
This can be particularly significant on queries without constants: such queries
tend to have more answers and thus require exploring larger proofs.

Both relevance analysis and magic sets try to identify proofs deriving a goal.
Magic sets do this ``at runtime'': the transformed rules derive only facts that
would be explored by top-down reasoning. In contrast, relevance analysis
identifies rules that can participate in such proofs ``offline'' by checking
whether all body atoms of a rule are derivable. Combining the two optimizations
is particularly effective at reducing the overheads described in the previous
paragraph.

\section{Conclusion \& Outlook}\label{sec:conclusion}

We presented a novel approach for goal-driven query answering over terminating
existential rules. Our empirical results clearly show that our technique can
lead to significant performance improvements and can mean the difference
between success and failure to answer a query. In the future, we shall
investigate the use of magic sets for query answering on nonterminating, but
decidable classes of existential rules (e.g., guarded, linear, or sticky). We
shall also consider adding optimizations such as tabling and subsumption.

\section*{Acknowledgments}

This research was funded by the EPSRC projects DBOnto, ED$^3$, MaSI$^3$, PDQ,
and the Alan Turing Institute.

\bibliographystyle{aaai}
\bibliography{references}

\iftoggle{withappendix}{
    \newpage
    \onecolumn
    \appendix
    \counterwithin{theorem}{section}
    \counterwithin{proposition}{section}

    \section{Full Experimental Results}\label{sec:full_results}

We conducted all experiments on a MacBook Pro laptop with a 3.1 GHz Intel Core
i7 processor, 16~GB of DDR3 RAM, and an 500~GB SSD, running macOS Sierra
Version 10.12.5.

Figure~\ref{figure:full_results} shows a \emph{cumulative distribution} for the
query times, the number of derived facts, and the number of rules in program
$P_6$ from Algorithm~\ref{alg:answer-query}. For example, the graphs in the
first column show how many test queries can be answered in a given time; thus,
in the $\LUBMonehun$ scenario, $\allrun$ answers 11 test queries in under
\SI{3}{\second}, and that this covers most test queries. Lines not reaching the
top of a graph show queries that could not be answered.

\begin{figure}[!p]
\centering
\scalebox{1.1}{
\begin{tabular}{cccc}
                                                & Query Times (ms)                          & \# Derived Facts                          & \# Rules \\[1ex]
    \rotatebox{90}{\hspace{0.8cm}\LUBMonehun}   & \begin{tikzpicture}[xscale=0.25, yscale=0.20]
\pgfplotsset{compat=newest}

\begin{axis}[
    width=15cm,
    height=15cm,
    enlargelimits=false,
    xmode=log,
    xtick pos=left,
    xlabel={\huge Time in ms},
    ymin=0,
    ymax=14,
    ytick pos=left,
    ylabel={\huge \#Queries},
    tick label style={font=\huge},
    tick align=outside,
    legend cell align=left,
    legend style={at={(0,1)}, anchor=north west, font=\huge}
]

\addplot+[ycomb, color=black, mark=star, mark size=3pt] table {
33938 14
};
\addlegendentry{$\materialization$}

\addplot+[color=green, mark=star, mark size=3pt] table {
859 1
1220 2
1619 3
2412 4
2448 5
2490 6
2672 7
2783 8
2817 9
2933 10
2977 11
7697 12
15163 13
93560 14
};
\addlegendentry{$\magicrun$}

\addplot+[color=blue, mark=square, mark size=3pt]  table {
958 1
1702 2
2799 3
23924 4
24565 5
24840 6
24866 7
24931 8
25247 9
25728 10
26172 11
26240 12
26252 13
26304 14
};
\addlegendentry{$\relevancerun$}

\addplot+[color=red, mark=+, mark size=3pt]  table {
862 1
1363 2
1598 3
2533 4
2587 5
2645 6
2650 7
2803 8
2832 9
3084 10
3229 11
8530 12
14090 13
77146 14
};
\addlegendentry{$\allrun$}

\end{axis}
\end{tikzpicture}        & \begin{tikzpicture}[xscale=0.25, yscale=0.20]
\pgfplotsset{compat=newest}

\begin{axis}[
    width=15cm,
    height=15cm,
    enlargelimits=false,
    xmode=log,
    xtick pos=left,
    xlabel={\huge \#Facts},
    ymin=0,
    ymax=14,
    ytick pos=left,
    ylabel={\huge \#Queries},
    tick label style={font=\huge},
    tick align=outside,
    legend cell align=left,
    legend style={at={(0,1)}, anchor=north west, font=\huge}
]

\addplot+[ycomb, color=black, mark=star, mark size=3pt] table {
23656301 14
};
\addlegendentry{$\materialization$}

\addplot+[color=green, mark=star, mark size=3pt] table {
18 1
50 2
340 3
1391 4
4115 5
5437 6
43880 7
68318 8
286914 9
519441 10
1591942 11
3027970 12
8147668 13
59237184 14
};
\addlegendentry{$\magicrun$}

\addplot+[color=blue, mark=square, mark size=3pt]  table {
92339 1
1591940 2
2250931 3
17373561 4
17373561 5
17373624 6
17374029 7
17374276 8
17400804 9
17435651 10
17435900 11
17443426 12
18422089 13
19522421 14
};
\addlegendentry{$\relevancerun$}

\addplot+[color=red, mark=+, mark size=3pt]  table {
18 1
44 2
293 3
1391 4
3863 5
4945 6
39342 7
62107 8
271899 9
464027 10
1591942 11
3027970 12
7637082 13
52411625 14
};
\addlegendentry{$\allrun$}

\end{axis}
\end{tikzpicture}       & \begin{tikzpicture}[xscale=0.25, yscale=0.20]
\pgfplotsset{compat=newest}

\begin{axis}[
    width=15cm,
    height=15cm,
    enlargelimits=false,
    xmode=log,
    xtick pos=left,
    xlabel={\huge \#Rules},
    ymin=0,
    ymax=14,
    ytick pos=left,
    ylabel={\huge \#Queries},
    tick label style={font=\huge},
    tick align=outside,
    legend cell align=left,
    legend style={at={(0,1)}, anchor=north west, font=\huge}
]

\addplot+[ycomb, color=black, mark=star, mark size=3pt] table {
144 14
};
\addlegendentry{$\materialization$}

\addplot+[color=green, mark=star, mark size=3pt] table {
4 1
12 2
26 3
247 4
253 6
255 7
256 8
257 9
258 11
259 12
260 13
275 14
};
\addlegendentry{$\magicrun$}

\addplot+[color=blue, mark=square, mark size=3pt]  table {
2 1
4 2
13 3
91 8
92 11
93 13
98 14
};
\addlegendentry{$\relevancerun$}

\addplot+[color=red, mark=+, mark size=3pt]  table {
4 1
8 2
19 3
172 4
176 5
180 6
181 9
182 10
185 11
186 12
189 13
223 14
};
\addlegendentry{$\allrun$}

\end{axis}
\end{tikzpicture} \\
    \rotatebox{90}{\hspace{0.9cm}\LUBMonek}     & \begin{tikzpicture}[xscale=0.25, yscale=0.20]
\pgfplotsset{compat=newest}

\begin{axis}[
    width=15cm,
    height=15cm,
    enlargelimits=false,
    xmode=log,
    xtick pos=left,
    xlabel={\huge Time in ms},
    ymin=0,
    ymax=14,
    ytick pos=left,
    ylabel={\huge \#Queries},
    tick label style={font=\huge},
    tick align=outside,
    legend cell align=left,
    legend style={at={(0,1)}, anchor=north west, font=\huge}
]

\addplot+[color=green, mark=star, mark size=3pt] table {
9640 1
14841 2
19737 3
33548 4
33996 5
35741 6
36584 7
37912 8
38185 9
39457 10
40060 11
};
\addlegendentry{$\magicrun$}

\addplot+[color=red, mark=+, mark size=3pt]  table {
8652 1
14592 2
21763 3
33254 4
34840 5
35288 6
36171 7
36811 8
37046 9
40719 10
44898 11
};
\addlegendentry{$\allrun$}

\end{axis}
\end{tikzpicture}        & \begin{tikzpicture}[xscale=0.25, yscale=0.20]
\pgfplotsset{compat=newest}

\begin{axis}[
    width=15cm,
    height=15cm,
    enlargelimits=false,
    xmode=log,
    xtick pos=left,
    xlabel={\huge \#Facts},
    ymin=0,
    ymax=14,
    ytick pos=left,
    ylabel={\huge \#Queries},
    tick label style={font=\huge},
    tick align=outside,
    legend cell align=left,
    legend style={at={(0,1)}, anchor=north west, font=\huge}
]

\addplot+[color=green, mark=star, mark size=3pt] table {
18 1
50 2
340 3
1391 4
4115 5
5437 6
43880 7
68318 8
519441 9
2929103 10
15849532 11
};
\addlegendentry{$\magicrun$}

\addplot+[color=red, mark=+, mark size=3pt]  table {
18 1
44 2
293 3
1391 4
3863 5
4945 6
39342 7
62107 8
464027 9
2777769 10
15849532 11
};
\addlegendentry{$\allrun$}

\end{axis}
\end{tikzpicture}       & \begin{tikzpicture}[xscale=0.25, yscale=0.20]
\pgfplotsset{compat=newest}

\begin{axis}[
    width=15cm,
    height=15cm,
    enlargelimits=false,
    xmode=log,
    xmax=300,
    xtick pos=left,
    xlabel={\huge \#Rules},
    ymin=0,
    ymax=14,
    ytick pos=left,
    ylabel={\huge \#Queries},
    tick label style={font=\huge},
    tick align=outside,
    legend cell align=left,
    legend style={at={(0,1)}, anchor=north west, font=\huge}
]

\addplot+[ycomb, color=black, mark=star, mark size=3pt] table {
144 14
};
\addlegendentry{$\materialization$}

\addplot+[color=green, mark=star, mark size=3pt] table {
4 1
12 2
26 3
253 5
255 6
256 7
257 8
258 10
260 11
};
\addlegendentry{$\magicrun$}

\addplot+[color=red, mark=+, mark size=3pt]  table {
4 1
8 2
19 3
176 4
180 5
181 8
182 9
186 10
189 11
};
\addlegendentry{$\allrun$}

\end{axis}
\end{tikzpicture} \\
    \rotatebox{90}{\hspace{0.9cm}\Deepthreehun} & \begin{tikzpicture}[xscale=0.25, yscale=0.20]
\pgfplotsset{compat=newest}

\begin{axis}[
    width=15cm,
    height=15cm,
    enlargelimits=false,
    xmode=log,
    xtick pos=left,
    xlabel={\huge Time in ms},
    ymin=0,
    ymax=20,
    ytick pos=left,
    ylabel={\huge \#Queries},
    tick label style={font=\huge},
    tick align=outside,
    legend cell align=left,
    legend style={at={(0,1)}, anchor=north west, font=\huge}
]

\addplot+[color=green, mark=star, mark size=3pt] table {
837 1
3143 2
12107 3
13087 4
16437 5
16642 6
16918 7
17551 8
19533 9
23249 10
30952 11
31441 12
36315 13
41470 14
41642 15
52831 16
149816 17
368324 18
556665 19
};
\addlegendentry{$\magicrun$}

\addplot+[color=blue, mark=square, mark size=3pt]  table {
32 1
4215 2
4273 3
4301 4
4590 5
4941 6
5024 7
5046 8
5071 9
5085 10
5199 11
5528 12
5541 13
5879 14
6377 15
6943 16
7304 17
10911 18
16848 19
};
\addlegendentry{$\relevancerun$}

\addplot+[color=red, mark=+, mark size=3pt]  table {
112 1
4215 2
4355 3
4365 4
4722 5
5103 6
5190 7
5220 8
5241 9
5258 10
5475 11
5586 12
5915 13
6450 14
6878 15
9337 16
9887 17
14044 18
17931 19
};
\addlegendentry{$\allrun$}

\end{axis}
\end{tikzpicture}        & \begin{tikzpicture}[xscale=0.25, yscale=0.20]
\pgfplotsset{compat=newest}

\begin{axis}[
    width=15cm,
    height=15cm,
    enlargelimits=false,
    xmin=1,
    xmode=log,
    xtick pos=left,
    xlabel={\huge \#Facts},
    ymin=0,
    ymax=20,
    ytick pos=left,
    ylabel={\huge \#Queries},
    tick label style={font=\huge},
    tick align=outside,
    legend cell align=left,
    legend style={at={(0,1)}, anchor=north west, font=\huge}
]

\addplot+[color=green, mark=star, mark size=3pt] table {
529 1
574 2
640 3
689 4
1059 5
1219 6
1521 7
1720 8
937102 9
4070721 10
4304941 11
4988143 12
7901323 13
8980243 14
8980259 15
8981103 16
30422241 17
73925191 18
106190116 19
};
\addlegendentry{$\magicrun$}

\addplot+[color=blue, mark=square, mark size=3pt]  table {
1 1
41 2
45 3
61 5
70 6
80 7
111 8
143 9
154 10
286 11
289 12
365 13
877 14
1545 15
1912 16
4883 17
4980 18
13016 19
};
\addlegendentry{$\relevancerun$}

\addplot+[color=red, mark=+, mark size=3pt]  table {
1 1
95 2
131 3
143 4
150 5
157 6
191 7
210 8
230 9
241 10
384 11
392 12
796 13
826 14
915 15
917 16
1505 17
10628 18
39545 19
};
\addlegendentry{$\allrun$}

\end{axis}
\end{tikzpicture}       & \begin{tikzpicture}[xscale=0.25, yscale=0.20]
\pgfplotsset{compat=newest}

\begin{axis}[
    width=15cm,
    height=15cm,
    enlargelimits=false,
    xmin=1,
    xmax=10000,
    xmode=log,
    xtick pos=left,
    xlabel={\huge \#Rules},
    ymin=0,
    ymax=20,
    ytick pos=left,
    ylabel={\huge \#Queries},
    tick label style={font=\huge},
    tick align=outside,
    legend cell align=left,
    legend style={at={(0,1)}, anchor=north west, font=\huge}
]

\addplot+[ycomb, color=black, mark=star, mark size=3pt] table {
4842 20
};
\addlegendentry{materialization}

\addplot+[color=green, mark=star, mark size=3pt] table {
355 1
1164 2
2104 3
2195 4
2209 5
2980 6
3294 7
3617 8
3621 9
3625 10
3659 11
3734 12
3888 13
3918 14
4168 15
4272 16
4289 17
4888 18
4891 19
};
\addlegendentry{$\magicrun$}

\addplot+[color=blue, mark=square, mark size=3pt]  table {
1 1
33 2
38 3
41 4
46 5
53 6
56 7
69 8
76 9
102 10
133 11
174 12
197 13
296 14
579 15
623 16
673 17
1109 18
1130 19
1381 20
};
\addlegendentry{$\relevancerun$}

\addplot+[color=red, mark=+, mark size=3pt]  table {
1 1
37 2
43 3
48 4
52 5
62 6
65 7
78 8
88 9
117 10
150 11
194 12
226 13
354 14
754 15
821 16
827 17
1418 18
1573 19
1608 20
};
\addlegendentry{$\allrun$}

\end{axis}
\end{tikzpicture} \\
    \rotatebox{90}{\hspace{0.5cm}\Doctorsonem}  & \begin{tikzpicture}[xscale=0.25, yscale=0.20]
\pgfplotsset{compat=newest}

\begin{axis}[
    width=15cm,
    height=15cm,
    enlargelimits=false,
    xmode=log,
    xtick pos=left,
    xlabel={\huge Time in ms},
    ymin=0,
    ymax=18,
    ytick pos=left,
    ylabel={\huge \#Queries},
    tick label style={font=\huge},
    tick align=outside,
    legend cell align=left,
    legend style={at={(0,1)}, anchor=north west, font=\huge}
]

\addplot+[ycomb, color=black, mark=star, mark size=3pt] table {
23034 18
};
\addlegendentry{$\materialization$}

\addplot+[color=green, mark=star, mark size=3pt] table {
4278 1
4380 2
4469 3
4751 4
5571 5
5655 6
5963 7
6474 8
6541 9
7991 10
8508 11
47265 12
47462 13
47888 14
54861 15
65802 16
68138 17
69209 18
};
\addlegendentry{$\magicrun$}

\addplot+[color=blue, mark=square, mark size=3pt]  table {
13762 1
13902 2
13967 3
14068 4
14108 5
14264 6
14278 7
14358 8
14460 9
14578 10
15210 11
15582 12
15676 13
16048 14
16118 15
16945 16
18960 17
22221 18
};
\addlegendentry{$\relevancerun$}

\addplot+[color=red, mark=+, mark size=3pt]  table {
1162 1
1477 2
1780 3
1812 4
1848 5
2111 6
2244 7
2387 8
2449 9
2472 10
2713 11
25575 12
28810 13
28953 14
29823 15
29879 16
34113 17
37815 18
};
\addlegendentry{$\allrun$}

\end{axis}
\end{tikzpicture}      & \begin{tikzpicture}[xscale=0.25, yscale=0.20]
\pgfplotsset{compat=newest}

\begin{axis}[
    width=15cm,
    height=15cm,
    enlargelimits=false,
    xmode=log,
    xtick pos=left,
    xlabel={\huge \#Facts},
    ymin=0,
    ymax=18,
    ytick pos=left,
    ylabel={\huge \#Queries},
    tick label style={font=\huge},
    tick align=outside,
    legend cell align=left,
    legend style={at={(0,1)}, anchor=north west, font=\huge}
]

\addplot+[ycomb, color=black, mark=star, mark size=3pt] table {
792000 18
};
\addlegendentry{$\materialization$}

\addplot+[color=green, mark=star, mark size=3pt] table {
5940 1
6403 2
6851 3
52075 4
263152 5
263163 7
264283 8
264292 9
264296 10
264303 11
6658002 12
7446002 13
7447002 15
9033002 16
9041002 17
9822002 18
};
\addlegendentry{$\magicrun$}

\addplot+[color=blue, mark=square, mark size=3pt]  table {
952501 7
953408 8
953472 9
953500 12
953631 14
1742500 18
};
\addlegendentry{$\relevancerun$}

\addplot+[color=red, mark=+, mark size=3pt]  table {
15 2
917 3
1374 4
1469 5
2957 6
2959 7
2963 8
2965 9
4087 11
1908002 13
1910002 14
2697002 18
};
\addlegendentry{$\allrun$}

\end{axis}
\end{tikzpicture}     & \begin{tikzpicture}[xscale=0.25, yscale=0.20]
\pgfplotsset{compat=newest}

\begin{axis}[
    width=15cm,
    height=15cm,
    enlargelimits=false,
    xmode=log,
    xtick pos=left,
    xlabel={\huge \#Rules},
    ymin=0,
    ymax=18,
    ytick pos=left,
    ylabel={\huge \#Queries},
    tick label style={font=\huge},
    tick align=outside,
    legend cell align=left,
    legend style={at={(0.35,0)}, anchor=south west, font=\huge},
]

\addplot+[ycomb, color=black, mark=star, mark size=3pt] table {
16 18
};
\addlegendentry{$\materialization$}

\addplot+[color=green, mark=star, mark size=3pt] table {
80 1
82 4
116 7
122 8
123 12
124 16
193 17
194 18
};
\addlegendentry{$\magicrun$}

\addplot+[color=blue, mark=square, mark size=3pt]  table {
6 13
8 18
};
\addlegendentry{$\relevancerun$}

\addplot+[color=red, mark=+, mark size=3pt]  table {
10 15
24 16
28 17
30 18
};
\addlegendentry{$\allrun$}

\end{axis}
\end{tikzpicture} \\
    \rotatebox{90}{\hspace{1.1cm}\STB}          & \begin{tikzpicture}[xscale=0.25, yscale=0.20]
\pgfplotsset{compat=newest}

\begin{axis}[
    width=15cm,
    height=15cm,
    enlargelimits=false,
    xmode=log,
    xtick pos=left,
    xlabel={\huge Time in ms},
    ymin=0,
    ymax=50,
    ytick pos=left,
    ylabel={\huge \#Queries},
    tick label style={font=\huge},
    tick align=outside,
    legend cell align=left,
    legend style={at={(0,1)}, anchor=north west, font=\huge}
]

\addplot+[ycomb, color=black, mark=star, mark size=3pt] table {
25149 50
};
\addlegendentry{$\materialization$}

\addplot+[color=green, mark=star, mark size=3pt] table {
12865 1
13065 2
13094 3
13131 4
13227 5
13409 6
13624 7
13820 8
13836 9
13865 10
14098 11
14208 12
14315 13
14335 14
14501 15
14653 16
14862 17
15078 18
15416 19
15700 20
15937 21
16205 22
16243 23
16393 24
16423 25
16445 26
17233 27
17700 28
17922 29
18772 30
23185 31
23428 32
23682 33
25866 34
27001 35
27673 37
28662 38
30070 39
30943 40
31021 41
31491 42
31760 43
32119 44
32486 45
32699 46
33599 47
34976 48
38227 49
38345 50
};
\addlegendentry{$\magicrun$}

\addplot+[color=blue, mark=square, mark size=3pt]  table {
470 1
471 2
475 3
478 4
516 5
521 6
585 7
633 8
659 9
682 10
712 11
755 12
765 13
768 14
816 15
862 16
889 17
896 18
916 19
927 20
929 21
930 22
965 23
1067 24
1072 25
1080 26
1089 27
1113 28
1162 29
1198 30
1217 31
1240 33
1254 34
1259 35
1422 36
3935 37
4041 38
18729 39
20076 40
142021 41
};
\addlegendentry{$\relevancerun$}

\addplot+[color=red, mark=+, mark size=3pt]  table {
470 1
474 2
479 3
482 4
488 5
529 6
535 7
546 8
549 9
557 10
562 11
564 12
567 13
569 14
570 15
576 16
580 18
582 19
584 20
590 21
611 22
623 23
656 24
706 25
722 26
731 27
754 28
763 29
794 30
832 31
880 32
923 33
999 35
1010 36
3745 37
3752 38
18440 39
20044 40
141656 41
};
\addlegendentry{$\allrun$}

\end{axis}
\end{tikzpicture}         & \begin{tikzpicture}[xscale=0.25, yscale=0.20]
\pgfplotsset{compat=newest}

\begin{axis}[
    width=15cm,
    height=15cm,
    enlargelimits=false,
    xmin=1,
    xmode=log,
    xtick pos=left,
    xlabel={\huge \#Facts},
    ymin=0,
    ymax=50,
    ytick pos=left,
    ylabel={\huge \#Queries},
    tick label style={font=\huge},
    tick align=outside,
    legend cell align=left,
    legend style={at={(0,1)}, anchor=north west, font=\huge}
]

\addplot+[ycomb, color=black, mark=star, mark size=3pt] table {
1918217 50
};
\addlegendentry{$\materialization$}

\addplot+[color=green, mark=star, mark size=3pt] table {
720088 1
720206 2
723188 3
723307 4
729400 5
729518 6
737894 7
738014 8
759266 9
779895 10
790866 11
791344 12
1044564 13
1044682 14
1065656 15
1065774 16
1889492 17
1899491 18
1929485 19
1929655 20
2596783 21
2596901 22
2608757 23
2608866 24
2659602 25
2659720 26
2667587 27
2667874 28
2928958 29
2929633 30
5970945 31
6031305 32
6099053 33
6220922 34
6315147 35
6380996 36
7302733 37
7447916 38
7769682 39
7957661 40
7978970 41
7983536 42
8006785 43
8091001 44
9205386 45
10491637 46
10491637 47
10743313 48
10760162 49
10803978 50
};
\addlegendentry{$\magicrun$}

\addplot+[color=blue, mark=square, mark size=3pt]  table {
1 5
29928 6
29988 7
29991 8
29994 9
29997 10
30075 11
39617 12
39617 13
39745 14
39913 15
39913 16
39994 17
39995 18
39996 19
40011 20
49512 21
49512 22
49552 23
50892 24
59995 25
60059 26
69889 27
69889 28
69931 29
79748 30
79748 31
79749 32
79975 33
79975 34
80026 35
89985 36
89985 37
89998 38
89998 39
90020 40
90051 41
};
\addlegendentry{$\relevancerun$}

\addplot+[color=red, mark=+, mark size=3pt]  table {
1 5
9 7
10 9
13 10
15 12
16 14
19 15
21 16
25 17
28 19
36 21
43 23
50 24
51 25
20546 26
20637 27
20833 28
30555 29
39944 30
39997 31
40057 32
40455 33
40686 34
49930 35
49982 36
49987 37
49992 38
60382 39
70996 40
90643 41
};
\addlegendentry{$\allrun$}

\end{axis}
\end{tikzpicture}        & \begin{tikzpicture}[xscale=0.25, yscale=0.20]
\pgfplotsset{compat=newest}

\begin{axis}[
    width=15cm,
    height=15cm,
    enlargelimits=false,
    xmode=log,
    xtick pos=left,
    xlabel={\huge \#Rules},
    ymin=0,
    ymax=50,
    ytick pos=left,
    ylabel={\huge \#Queries},
    tick label style={font=\huge},
    tick align=outside,
    legend cell align=left,
    legend style={at={(0.35,0)}, anchor=south west, font=\huge},
]

\addplot+[ycomb, color=black, mark=star, mark size=3pt] table {
425 50
};
\addlegendentry{$\materialization$}

\addplot+[color=green, mark=star, mark size=3pt] table {
1804 1
1822 2
1823 3
1825 4
1830 5
1832 6
1834 8
1837 9
1838 10
1839 11
1840 12
1844 13
1845 14
1847 17
1848 20
1849 24
1851 25
1852 27
1853 28
1854 29
1855 30
1856 32
1857 35
1858 36
1859 38
2804 39
2806 40
2807 41
2808 44
2811 45
2812 46
2816 48
2817 49
2818 50
};
\addlegendentry{$\magicrun$}

\addplot+[color=blue, mark=square, mark size=3pt]  table {
1 5
3 9
4 11
5 21
6 24
7 26
8 29
9 35
10 41
};
\addlegendentry{$\relevancerun$}

\addplot+[color=red, mark=+, mark size=3pt]  table {
1 5
6 9
8 11
9 17
10 21
11 24
14 26
15 29
16 32
18 35
19 41
};
\addlegendentry{$\allrun$}

\end{axis}
\end{tikzpicture} \\
    \rotatebox{90}{\hspace{1.1cm}\Ont}          & \begin{tikzpicture}[xscale=0.25, yscale=0.20]
\pgfplotsset{compat=newest}

\begin{axis}[
    width=15cm,
    height=15cm,
    enlargelimits=false,
    xmode=log,
    xtick pos=left,
    xlabel={\huge Time in ms},
    ymin=0,
    ymax=42,
    ytick={0,5,...,42},
    ytick pos=left,
    ylabel={\huge \#Queries},
    tick label style={font=\huge},
    tick align=outside,
    legend cell align=left,
    legend style={at={(0.065,0)}, anchor=south west, font=\huge}
]

\addplot+[ycomb, color=black, mark=star, mark size=3pt] table {
88260 42
};
\addlegendentry{$\materialization$}

\addplot+[color=green, mark=star, mark size=3pt] table {
37269 1
40125 2
60360 3
657940 4
678630 5
682456 6
682614 7
684072 8
684264 9
687622 10
694018 11
695326 12
695762 13
696137 14
703285 15
703669 16
715796 17
715883 18
715989 19
716179 20
727647 21
729265 22
729698 23
734481 24
741944 25
748388 26
748472 27
750732 28
759876 29
771129 30
779439 31
783851 32
791962 33
800829 34
802100 35
808180 36
816461 37
816693 38
823223 39
849373 40
886593 41
1498437 42
};
\addlegendentry{$\magicrun$}

\addplot+[color=blue, mark=square, mark size=3pt]  table {
1254 1
1272 2
1273 3
1280 4
1304 5
1491 6
1542 7
1568 8
1593 10
1631 11
1639 12
1643 13
1656 14
1671 15
1680 16
1696 17
1700 18
1701 19
1713 21
1716 22
1731 23
1741 24
1780 25
1829 26
1833 27
1860 28
1895 29
1904 30
1912 31
1920 32
1932 33
1949 34
1967 35
1979 36
2074 37
2328 38
6235 39
6334 40
6564 41
10053 42
};
\addlegendentry{$\relevancerun$}

\addplot+[color=red, mark=+, mark size=3pt]  table {
1272 1
1278 2
1290 3
1303 4
1344 5
1447 6
1476 7
1479 8
1496 10
1506 11
1510 12
1519 13
1520 14
1521 15
1540 16
1574 17
1577 18
1610 19
1630 21
1638 22
1681 23
1721 24
1774 25
1948 26
1956 27
1964 28
1979 29
1983 30
1985 31
2011 32
2046 33
2081 34
2099 35
2103 36
2140 37
2146 38
5749 39
6042 40
6224 41
10188 42
};
\addlegendentry{$\allrun$}

\end{axis}
\end{tikzpicture}    & \begin{tikzpicture}[xscale=0.25, yscale=0.20]
\pgfplotsset{compat=newest}

\begin{axis}[
    width=15cm,
    height=15cm,
    enlargelimits=false,
    xmin=1,
    xmode=log,
    xtick pos=left,
    xlabel={\huge \#Facts},
    ymin=0,
    ymax=42,
    ytick={0,5,...,42},
    ytick pos=left,
    ylabel={\huge \#Queries},
    tick label style={font=\huge},
    tick align=outside,
    legend cell align=left,
    legend style={at={(0,1)}, anchor=north west, font=\huge}
]

\addplot+[ycomb, color=black, mark=star, mark size=3pt] table {
5674103 42
};
\addlegendentry{$\materialization$}

\addplot+[color=green, mark=star, mark size=3pt] table {
1597701 1
1660650 2
6357543 3
224218553 4
224219369 5
224220182 6
224220183 7
224220185 8
224234898 9
237878993 10
238689982 11
239155136 12
239466496 13
239512299 14
240582863 15
244174206 16
245595781 17
246379341 18
251567809 19
251660861 20
252733681 21
253372621 22
253490783 23
253490791 24
253911484 25
254199567 26
254655815 27
255752185 28
260142460 29
260684440 30
260961336 31
265239521 32
266253577 33
269456375 34
273790944 36
273814523 37
273814525 38
273961722 39
276218267 40
277740867 41
283684226 42
};
\addlegendentry{$\magicrun$}

\addplot+[color=blue, mark=square, mark size=3pt]  table {
1 5
29355 6
29453 7
29851 8
39120 9
39151 10
39227 11
39284 12
39285 13
39489 14
49031 15
49092 16
49261 17
49290 18
49373 19
50371 20
58809 21
58981 22
59198 23
59265 24
59308 25
68769 26
68791 27
68817 28
68869 29
68887 30
69066 31
70887 32
78672 33
84962 34
85009 35
88490 36
88552 37
98264 38
108239 39
118069 40
118077 41
118093 42
};
\addlegendentry{$\relevancerun$}

\addplot+[color=red, mark=+, mark size=3pt]  table {
1 5
9 6
10 7
11 9
12 10
14 11
17 12
41 13
19829 14
20054 15
20203 16
20247 17
20791 18
29534 19
29867 20
29970 21
30154 22
31126 23
40006 24
41777 25
50049 26
59692 27
68370 28
68455 29
68647 30
78608 31
78859 32
98014 33
98650 34
98737 35
98786 36
98864 37
99343 38
108510 39
108725 40
118593 41
147496 42
};
\addlegendentry{$\allrun$}

\end{axis}
\end{tikzpicture}   & \begin{tikzpicture}[xscale=0.25, yscale=0.20]
\pgfplotsset{compat=newest}

\begin{axis}[
    width=15cm,
    height=15cm,
    enlargelimits=false,
    xmode=log,
    xtick pos=left,
    xlabel={\huge \#Rules},
    ymin=0,
    ymax=42,
    ytick={0,5,...,42},
    ytick pos=left,
    ylabel={\huge \#Queries},
    tick label style={font=\huge},
    tick align=outside,
    legend cell align=left,
    legend style={at={(0.35,0)}, anchor=south west, font=\huge},
]

\addplot+[ycomb, color=black, mark=star, mark size=3pt] table {
1707 42
};
\addlegendentry{$\materialization$}

\addplot+[color=green, mark=star, mark size=3pt] table {
10085 1
10097 2
10101 3
10110 4
10112 5
10117 6
10120 7
10121 10
10122 12
10124 13
10125 14
10126 15
10127 17
10129 18
10130 19
10131 20
10132 23
10134 24
10136 25
10137 28
10140 29
10141 30
10142 31
10143 33
10145 37
10147 40
10150 41
15622 42
};
\addlegendentry{$\magicrun$}

\addplot+[color=blue, mark=square, mark size=3pt]  table {
1 5
4 6
5 14
6 16
7 23
8 24
9 31
10 37
13 38
15 41
16 42
};
\addlegendentry{$\relevancerun$}

\addplot+[color=red, mark=+, mark size=3pt]  table {
1 5
8 6
9 7
10 10
11 11
13 14
14 16
15 17
16 24
18 25
19 31
22 32
24 34
27 35
31 36
32 37
33 38
36 39
39 40
41 42
};
\addlegendentry{$\allrun$}

\end{axis}
\end{tikzpicture} \\
\end{tabular}}
\caption{Times, numbers of derived facts, and total numbers of rules}\label{figure:full_results}
\end{figure}
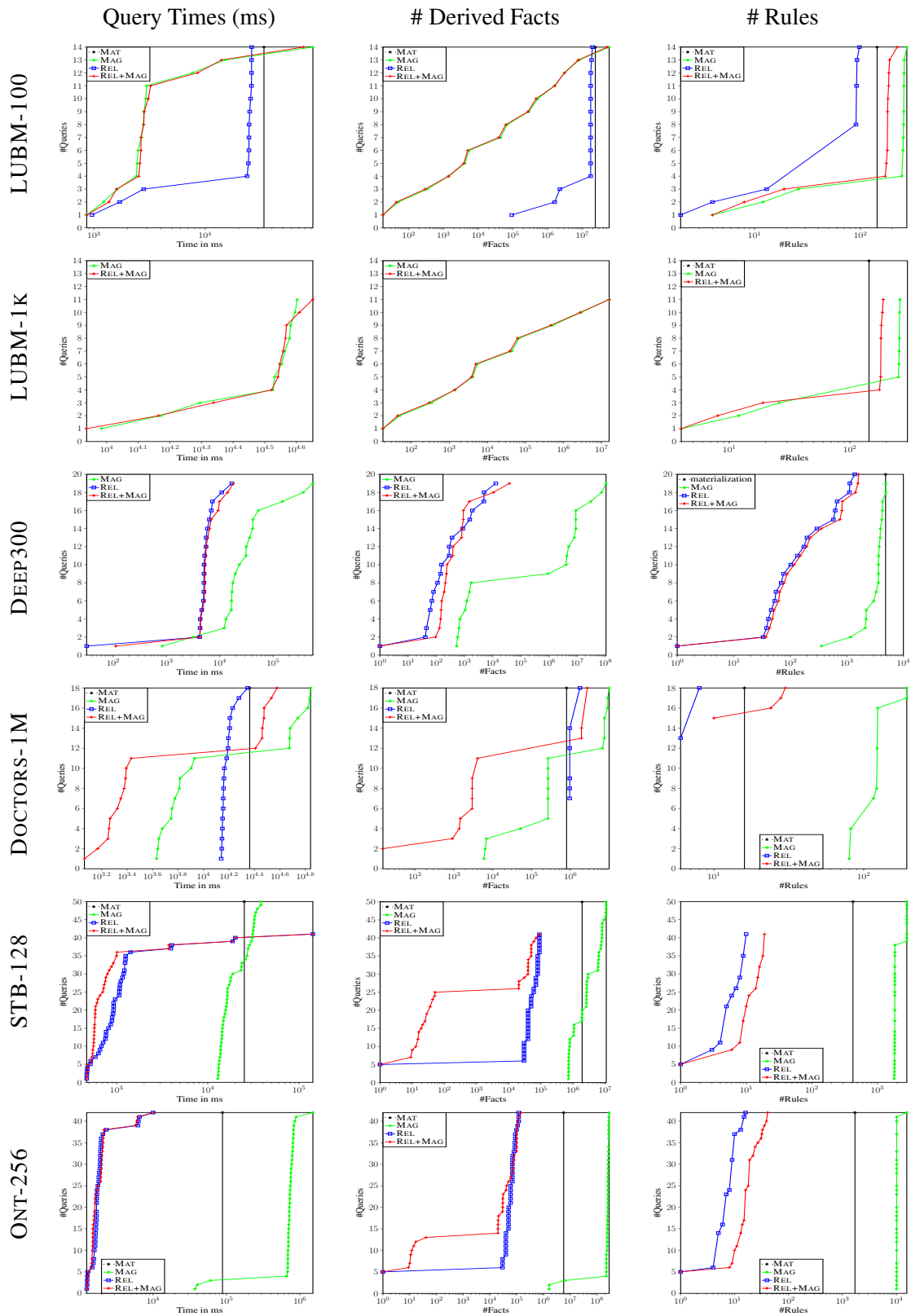

    \section{Proofs}

We now prove Theorem~\ref{thm:query-answering}, which shows that
Algorithm~\ref{alg:answer-query} correctly computes all query answers in a
finite amount of time. Towards this goal, we prove correctness of all
intermediate steps as well.

\subsection{Properties of Singularization}

Recall that our pipeline in Algorithm~\ref{alg:answer-query} begins with the
singularization transformation, Step \ref{alg:answer-query:sg}. This introduces
explicit equality atoms in the body of rules, as described in Section
\ref{sec:singularization}. Proposition~\ref{prop:sg} summarizes the properties
of singularization. \citet{tencate2009} proved the first two claims when UNA
holds, but extending their argument in the absence of UNA is straightforward.
The third property follows from the first two and the adjustments in
\eqref{eq:sing:predQ} to the existential rules defining the query predicate
$\predQ$.

\begin{proposition}\label{prop:sg}
    For each set $\Sigma$ of existential rules, base instance $B$, predicate
    $R$, tuples of constants ${\vec a}$ and variables ${\vec x}$, constants $b$
    and $c$, and ${\Sigma' = \sg{\Sigma} \cup \Ref{\Sigma} \cup \SymTrans}$,
    \begin{itemize}
        \item ${\Sigma \cup B \modelsEq R(\vec a)}$ iff ${\Sigma' \cup B
        \models \exists \vec x. R(\vec x) \wedge \bigwedge_{i=1}^n x_i \equals
        a_i}$;
        
        \item ${\Sigma \cup B \modelsEq b \equals c}$ iff ${\Sigma' \cup B
        \models b \equals c}$; and

        \item ${\Sigma \cup B \modelsEq \predQ(\vec a)}$ iff ${\Sigma' \cup B
        \models \predQ(\vec a)}$.
    \end{itemize}
\end{proposition}

\subsection{Correctness of the Relevance Analysis}

In Algorithm~\ref{alg:answer-query}, singularization is followed by the
standard Skolemization step. Next, relevance analysis is applied to prune the
set of rules, as described in Algorithm~\ref{alg:relevance} of Section
\ref{sec:relevance}. Our relevance algorithm is new, and hence we will need to
prove its correctness, which is captured in the following theorem.

\thmrelevance*

\begin{proof}
Let $B'$ be as chosen in line~\ref{alg:relevance:abstraction}; let $\eta$ be
any mapping of constants to constants such that ${\eta(B') \subseteq B}$ and
${\eta(c) = c}$ for each constant $c$ occurring in $P$; let $I$ and $P'$ be as
in line~\ref{alg:relevance:abstraction-fixpoint}; let $\Rint$ be the
intermediate set of rules as computed just before
line~\ref{alg:relevance:UNA:start}; and let $\Rfin$ and $\Dfin$ be the final
sets. Note that ${\Ref{\Rint} = \Ref{\Rfin}}$. Furthermore, ${\Rint \subseteq
P}$ so $\Rint$ is $\equals$-safe, and the transformation in
line~\ref{alg:relevance:UNA:desg} clearly preserves $\equals$-safety.

\medskip

($\Leftarrow$) For ${P_1 = \Rfin \cup \Ref{\Rfin} \cup \SymTrans}$, let ${I_0,
I_1, \dots}$ be the sequence of instances used to compute $\fixpoint{P_1}{B}$
as defined in Section~\ref{sec:preliminaries}. We show by induction $i$ that,
for each fact ${F \in I_i}$, we have ${P \cup \Ref{P} \cup \SymTrans \cup B
\models F}$. The base case is trivial. For the induction step, assume that the
claim holds for $I_i$, and consider applying a rule ${r \in \Rfin}$ with a
substitution $\sigma$ where ${\sigma(\head{r}) = F}$. Thus, we have
${\sigma(\body{r}) \subseteq I_i}$, and so ${P \cup \Ref{P} \cup \SymTrans \cup
B \models \sigma(\body{r})}$ holds by the inductive assumption. Rule ${r \in
\Rfin}$ is obtained from some rule ${r' \in \Rint \subseteq P}$ by
transformations in line~\ref{alg:relevance:UNA:desg}. Let $\sigma'$ be the
extension of $\sigma$ such that, for each variable $x$ that was replaced with a
term $t$, we set ${\sigma'(x) = \sigma(t)}$. Each body atom ${A' \in
\body{r'}}$ that was not removed in line~\ref{alg:relevance:UNA:desg}
corresponds to some ${A \in \body{r}}$ and ${\sigma'(A') = \sigma(A)}$ holds.
Moreover, since $P$ is $\equals$-safe, each body atom $A'$ of $r'$ that was
removed in line~\ref{alg:relevance:UNA:desg} is of the form ${x \equals y}$ or
${x \equals t}$, and variable $x$ occurs in some atom ${R_i(\dots,x,\dots) \in
\body{r'}}$; but then, the inductive assumption and the fact that $\Ref{\Rfin}$
contains the reflexivity rules for $B$ ensure ${P \cup \Ref{P} \cup \SymTrans
\cup B \models \sigma'(A')}$. Summarizing, we have ${P \cup \Ref{P} \cup
\SymTrans \cup B \models \sigma'(\body{r'})}$, and so ${P \cup \Ref{P} \cup
\SymTrans \cup B \models \sigma'(\head{r'})}$ holds, as required.

\medskip

($\Rightarrow$) For ${P_2 = P \cup \Ref{P} \cup \SymTrans}$, let ${I_0, I_1,
\dots}$ be the sequence of instances used to compute ${\fixpoint{P_2}{B}}$ as
defined in Section~\ref{sec:preliminaries}. We prove the claim in two steps.

First, we show by induction on $i$ that ${\eta(I_i) \subseteq I}$ holds. For
the base case, we have ${I_0 = B' \subseteq I}$, as required. Now assume that
${\eta(I_i) \subseteq I}$ holds, and consider an arbitrary fact ${F \in
I_{i+1}}$ derived by a rule ${r \in P'}$ and substitution $\sigma$; hence,
${\sigma(\body{r}) \subseteq I_i}$ holds. By the inductive assumption, we have
${\eta(\sigma(\body{r})) \subseteq \eta(I_i) \subseteq I}$. Moreover, $\eta$
does not affect the constants occurring in $P$, so ${\sigma'(\body{r})
\subseteq \eta(I_i) \subseteq I}$ holds where $\sigma'$ is the substitution
defined by ${\sigma'(x) = \eta(\sigma(x))}$ for each variable $x$ from the
domain of $\sigma$; consequently, we have ${\sigma'(\head{r}) \in I}$. Finally,
since $\eta$ does not affect the constants occurring in $P$, we have
${\sigma'(\head{r}) = \eta(F)}$, as required.

Second, we show by induction in $i$ that, for each fact ${F \in I_i}$ that is
not of the form ${t \equals t}$ and that satisfies ${\eta(F) \in \Dfin}$, we
have ${\Rfin \cup \Ref{\Rfin} \cup \SymTrans \cup B \models F}$.
Line~\ref{alg:relevance:init} of the algorithm ensures that ${\eta(\predQ(\vec
a)) \in \Dfin}$ holds for each ${\vec a}$ with ${\Rfin \cup \Ref{\Rfin} \cup
\SymTrans \cup B \models \predQ(\vec a)}$, so our claim holds.

The induction base holds trivially. Now assume that $I_i$ satisfies this
property and consider an arbitrary fact ${F \in I_{i+1}}$ derived by a rule ${r
\in P'}$ and substitution $\sigma$; hence, ${\sigma(\body{r}) \subseteq I_i}$
holds. Moreover, assume that $F$ is not of the form ${t \equals t}$; thus, ${r
\not\in \Ref{P}}$ and so ${r \in P \cup \SymTrans}$. Also, assume that
${\eta(F) \in \Dfin}$ holds; then ${\eta(F)}$ must have been added to
$\mathcal{T}$ in line~\ref{alg:relevance:add-T-D}, so it must have been
extracted from $\mathcal{T}$ at some point in
line~\ref{alg:relevance:choose-F}. Moreover, due to ${r \in P \cup \SymTrans}$,
rule $r$ was considered in line~\ref{alg:relevance:rule:start}. Finally, let
${r' = \sigma(r)}$; then ${\eta(\body{r'}) \subseteq I}$ holds by the previous
paragraph, which together with ${\eta(\head{r'}) = \eta(F)}$ and the fact that
$\eta$ does not affect the constants in $P$ ensures that substitution $\nu$
satisfying ${\nu(r) = \eta(r')}$ was considered in
line~\ref{alg:relevance:rule:start}. Consequently, $r$ is added to $\Rint$ in
line~\ref{alg:relevance:add-r}. We next prove the following property:
\begin{align}
    \Rfin \cup \Ref{\Rfin} \cup \SymTrans \cup B \models \body{r'}. \tag{$\ast$}
\end{align}

First, consider a ground atom ${F_i \in \body{r'}}$ not of the form ${t \equals
t}$. If atom $F_i$ is relational, or if $F_i$ contains function symbols, or if
${P \cup B}$ does not satisfy UNA, then $\eta(F_i)$ is not of the form ${c
\equals c}$ for $c$ a constant; hence, $\eta(F_i)$ added to $\Dfin$ in
line~\ref{alg:relevance:add-T-D}, and so the inductive assumption ensures
${\Rfin \cup \Ref{\Rfin} \cup \SymTrans \cup B \models F_i}$.

Second, consider a ground atom ${t \equals t \in \body{r'}}$. Rule $r$ is
$\equals$-safe, so atom ${t \equals t}$ is obtained from some ${A \in
\body{r}}$ of the form ${x \equals y}$ or ${x \equals t}$, and a relational
atom ${R_i(\vec t_i) \in \body{r}}$ exists such that ${\vars{A} \cap \vec t_i
\neq \emptyset}$. The previous paragraph ensures that ${\Rfin \cup \Ref{\Rfin}
\cup \SymTrans \cup B \models \sigma(R_i(\vec t_i))}$ holds; and ${r \in
\Rint}$ so $\Ref{\Rfin}$ contains all reflexivity rules for $R_i$; thus,
${\Rfin \cup \Ref{\Rfin} \cup \SymTrans \cup B \models t \equals t}$ holds.

Thus, property ($\ast$) holds. To complete the proof, we show that $\Rfin$
contains a rule that derives $F$. In particular, ${r \in \Rint}$ corresponds to
some rule ${r'' \in \Rfin}$ where the latter is obtained by transformations in
line~\ref{alg:relevance:UNA:desg}. Now consider an arbitrary body equality atom
${F_i \in \body{r'}}$. If $F_i$ is not of the form ${c \equals c}$ for some
constant $c$, then either ${P \cup B}$ does not satisfy UNA or $\eta(F_i)$ is
not of the form ${c \equals c}$ for $c$ a constant; hence, ${\langle r,i
\rangle}$ is added to $\mathcal{B}$ in line~\ref{alg:relevance:add-B}, and so
the body atom of $r''$ corresponding to $F_i$ is not eliminated in
line~\ref{alg:relevance:UNA:desg}. Thus, we have ${\sigma(\body{r''}) \subseteq
\body{r'}}$, so ${\Rfin \cup \Ref{\Rfin} \cup \SymTrans \cup B \models F}$.
\end{proof}

\subsection{Correctness of the Magic Transformation}

In line~\ref{alg:answer-query:magic} of Algorithm~\ref{alg:answer-query}, we
apply our variant of the magic sets algorithm to the result of the prior steps.
Note that both the input and output of this step are rules which may contain
equality atoms in both the head and the body. We need to prove not only that
this preserves the semantics of the input, but that it also preserves
$\equals$-safety. The latter guarantees that when, we remove the equalities in
the final two steps of the pipeline, we are left with a domain-independent set
of rules.

\thmmagic*

\begin{proof}
To see that $\mathcal{R}$ is $\equals$-safe, consider a rule ${r \in P \cup
\SymTrans}$ processed by the function $\mathsf{process}$. Since $r$ is
$\equals$-safe, the rule added to $\mathcal{R}$ in
line~\ref{alg:magic:mod-rule} for $r$ is clearly $\equals$-safe as well.
Moreover, the reordered atoms in line~\ref{alg:magic:reorder} satisfy the
condition from Definition~\ref{def:SIPS}:
\begin{itemize}
    \item each rule ${r \in P}$ can be ordered in such a way since $P$ is
    $\equals$-safe, and

    \item each rule ${r \in \SymTrans}$ can be ordered in such a way since
    $\ad{ff}$ is not a valid adornment of $\equals$, so ${\vec t^\beta}$ in
    line~\ref{alg:magic:reorder} contains at least one variable from the body
    of $r$.
\end{itemize}
Now consider a rule ${r' \in \mathcal{R}}$ computed from $r$ in
line~\ref{alg:magic:magic-rule}. For each atom ${R_i(\vec t_i) \in \body{r'}
\subseteq \body{r}}$ of the form ${x \equals y}$ or ${x \equals s}$, and for
${z \in \vars{\vec t_i}}$ the variable that satisfies the condition for
${R_i(\vec t_i)}$ in Definition~\ref{def:SIPS}, the $\equals$-safety of $r'$ is
ensured by ${\mgc{R}{\alpha}(\vec t^\beta) \in \body{r'}}$ if ${z \in T}$ for
${T = \vec t^\beta}$, and by ${R_j(\vec t_j) \in \body{r'}}$ if ${z \in \vec
t_j}$ for some ${j < i}$.

\medskip

We now proceed to prove that query answers remain preserved.

($\Leftarrow$) Each rule ${r' \in \mathcal{R}}$ where $\head{r'}$ does not
contain a magic predicate is obtained from some rule ${r \in P \cup \SymTrans}$
by appending an atom with a magic predicate to $\body{r}$. Thus, each
derivation from ${\mathcal{R} \cup \Ref{\mathcal{R}} \cup \SymTrans \cup B}$ of
a fact not containing a magic predicate corresponds directly to a derivation of
the fact from ${P \cup \Ref{P} \cup \SymTrans \cup B}$.

\medskip

($\Rightarrow$) Let ${P_1 = P \cup \Ref{P} \cup \SymTrans}$ and ${I =
\fixpoint{P_1}{B}}$, and let ${I_0, I_1, \dots}$ be the sequence of instances
used to compute $I$ as defined in Section~\ref{sec:preliminaries}. We prove by
induction on $i$ that, for each fact ${R(\vec s) \in I_i}$ not of the form ${s
\equals s}$, each magic predicate $\mgc{R}{\alpha}$ occurring in $\mathcal{R}$,
and each tuple of ground terms ${\vec s}$, the following two properties hold:
\begin{enumerate}
    \item ${\mathcal{R} \cup \Ref{\mathcal{R}} \cup \SymTrans \cup B \cup \{
    \mgc{\equals}{\adEqb}(\vec s^\ad{\beta}) \} \models {\equals}(\vec s)}$ for
    each ${\beta \in \{ \ad{bf}, \ad{fb} \}}$ if ${R = {\equals}}$ and ${\alpha
    = \adEqb}$;

    \item ${\mathcal{R} \cup \Ref{P'} \cup \SymTrans \cup B \cup \{
    \mgc{R}{\alpha}(\vec s^\alpha) \} \models R(\vec s)}$ if $R$ is \\ distinct
    from $\equals$ or ${\alpha \neq \adEqb}$.
\end{enumerate}
Line~\ref{alg:magic:init} ensures ${\mgc{\predQ}{\alpha} \leftarrow \; \in
\mathcal{R}}$ for ${\alpha = f \cdots f}$, so fact $\mgc{\predQ}{\alpha}$ is
derivable from $\mathcal{R}$ and $B$, and therefore, for each ground tuple of
terms ${\vec t}$ such that ${P \cup \Ref{P} \cup \SymTrans \cup B \models
\predQ(\vec t)}$, property 2 ensures ${\mathcal{R} \cup \Ref{\mathcal{R}} \cup
\SymTrans \cup B \models \predQ(\vec t)}$.

The base case holds trivially for each ${R(\vec s) \in I_0 = B}$. For the
induction step, assume that $I_i$ satisfies properties 1 and 2 and consider an
arbitrary fact ${R(\vec s) \in I_{i+1}}$ not of the form ${t \equals t}$
derived by a rule ${r \in P_1}$ and substitution $\sigma$; hence,
${\sigma(\body{r}) \subseteq I_i}$ and ${r \not\in \Ref{P}}$ hold. Moreover,
consider an arbitrary magic predicate $\mgc{R}{\alpha}$ occurring in
$\mathcal{R}$. Then, $\mgc{R}{\alpha}$ was extracted from $\mathcal{T}$ in
line~\ref{alg:magic:R:start}, and so rule ${r \in P \cup \SymTrans}$ was
considered in line~\ref{alg:magic:r:start}. Let ${R(\vec t)}$ be the head of
$r$; thus, we clearly have ${\sigma(\vec t) = \vec s}$.

Now consider ${\beta \in \{ \ad{bf}, \ad{fb} \}}$ if ${R = {\equals}}$ and
${\alpha = \adEqb}$, and let ${\beta = \alpha}$ otherwise. The algorithm calls
$\mathsf{process}(r,\alpha,\beta)$ in
lines~\ref{alg:magic:process:bf}--\ref{alg:magic:process:fb} or
line~\ref{alg:magic:process}. Let ${R_1(\vec t_1), \dots, R_n(\vec t_n)}$ be
the reordered body or $r$ from line~\ref{alg:magic:reorder}; as we have already
mentioned, such an ordering exists since $P$ is $\equals$-safe. We next prove
by another induction on ${1 \leq j \leq n}$ that the following property holds:
\begin{align}
    \mathcal{R} \cup \Ref{\mathcal{R}} \cup \SymTrans \cup B \cup \{ \mgc{R}{\alpha}(\vec s^\beta) \} \models \sigma(R_j(\vec t_j)).    \tag{$\ast$}
\end{align}    
The base case is the same as the induction step. Consider ${1 \leq j \leq n}$
such that the claim holds for each ${j' < j}$.

First, assume ${\sigma(R_j(\vec t_j))}$ is of the form ${t \equals t}$. Program
$P$ is $\equals$-safe, so $R_j(\vec t_j)$ is of the form ${x \equals y}$ or ${x
\equals t}$, and some ${z \in \vars{\vec t_j}}$ satisfies the condition of
Definition~\ref{def:SIPS}.
\begin{itemize}
    \item If ${z \in T}$ for ${T = \vec t^\beta}$, then ${z \in \vec t^\beta}$
    holds for the body atom of the form $\mgc{R}{\alpha}(\vec t^\beta)$ added
    in line~\ref{alg:magic:mod-rule} for $r$, so ${\sigma(z) = t \in \vec
    s^\beta}$. Moreover, $\mgc{R}{\alpha}(\vec s^\beta)$ holds by ($\ast$), and
    $\Ref{\mathcal{R}}$ contains the reflexivity rules for $\mgc{R}{\alpha}$,
    and one of them derives ${t \equals t}$.

    \item If ${z \in \vec t_{j'}}$ for some ${j' < j}$ and body atom
    $R_{j'}(\vec t_{j'})$ of $r$, then ($\ast$) ensures that
    $\sigma(R_{j'}(\vec t_{j'}))$ is derived. Moreover, $\Ref{\mathcal{R}}$
    contains the reflexivity rules for $R_{j'}$, and one of them derives ${t
    \equals t}$.
\end{itemize}

Second, assume ${\sigma(R_j(\vec t_j))}$ is not of the form ${t \equals t}$ and
that $R_j$ does not occur in the head of a rule in $P$. The rules from
${\Ref{\mathcal{R}} \cup \SymTrans}$ can derive only facts of the form ${t
\equals t}$, so therefore $R_j$ is different from $\equals$. But then we have
${\sigma(R_j(\vec t_j)) \in B}$, and thus ($\ast$) holds trivially.

Third, assume ${\sigma(R_j(\vec t_j))}$ is not of the form ${t \equals t}$ and
that $R_j$ is processed in
lines~\ref{alg:magic:body:start}--\ref{alg:magic:body:end}. Let $\gamma$ be the
adornment for ${R_j(\vec t_j)}$ in line~\ref{alg:magic:adorn} and let $S$ be
the magic predicate in lines~\ref{alg:magic:S:start}--\ref{alg:magic:S:end}.
The rule added to $P$ in line \ref{alg:magic:magic-rule} and ($\ast$) ensure
${\mathcal{R} \cup \Ref{\mathcal{R}} \cup \SymTrans \cup B \cup \{
\mgc{R}{\alpha}(\vec s^\beta) \} \models \sigma(S(\vec t_j^\gamma))}$, so the
inductive assumption for properties 1 and 2 ensures ${\mathcal{R} \cup
\Ref{\mathcal{R}} \cup \SymTrans \cup B \cup \{ \mgc{R}{\alpha}(\vec s^\beta)
\} \models \sigma(R_j(\vec t_j))}$.

This completes the proof of ($\ast$). To complete the proof of this theorem,
note that ($\ast$) and ${r \in \SymTrans}$ trivially imply property 1.
Moreover, if ${R = {\equals}}$ and ${\alpha = \adEqb}$, then ($\ast$) and the
rule added to $\mathcal{R}$ in line~\ref{alg:magic:mod-rule} ensure property 1.
Finally, in all other cases, ($\ast$) and the rule added to $\mathcal{R}$ in
line~\ref{alg:magic:mod-rule} ensure property 2.
\end{proof}

\subsection{Magic Transformation Preserves Chase Termination}

We next show that the magic sets transformation preserves finiteness of the
least fixpoint of the programs considered in our pipeline. Thus, applying the
chase for logic programs to the result is guaranteed to terminate, as mentioned
in Section~\ref{sec:magic}.

\thmtermination*

\begin{proof}
For $F$ a fact, let $\dep{F}$ be the depth of $F$ as usual (where constants
have depth zero). Let $M$ be the maximum depth of an atom in
${\fixpoint{P_1}{B}}$, and let ${I_0, I_1, \dots}$ be the sequence of instances
used to compute ${\fixpoint{P_2}{B}}$ as defined in
Section~\ref{sec:preliminaries}. We show by induction on $k$ that, for each
fact ${F \in I_k}$, we have ${\dep{F} \leq M}$. The base case is trivial, so
assume that the claim holds for some $I_k$ and consider an application of a
rule ${r \in P_2}$ to $I_k$ with some $\sigma$. If ${r \in \Ref{P} \cup
\SymTrans}$ or $r$ was added to $P_2$ in line~\ref{alg:magic:mod-rule}, then
$P_1$ contains a rule $r'$ such that ${\body{r'} \subseteq \body{r}}$; but
then, since $r'$ derives on $P_1$ and $B$ an atom of depth at most $M$, so does
$r$ on $P_2$ and $B$. Finally, assume that $r$ was added to $P_2$ in
line~\ref{alg:magic:magic-rule} for some $i$. By the inductive assumption, we
have ${\dep{\sigma(S(\vec t_i^\gamma))} \leq M}$ and ${\dep{\sigma(R_j(\vec
t_j))} \leq M}$ for ${1 \leq j < i}$. Moreover, the body of $r$ does not
contain function symbols, and so the head atom ${\head{r'} = S(\vec
t_i^\gamma)}$ of $r'$ does not contain a function symbol either, and all
variables in $\head{r'}$ occur in $\body{r'}$. Clearly, we have
${\dep{\sigma(\head{r'})} \leq M}$, as required.
\end{proof}

\subsection{Correctness of Eliminating Function Symbols from Bodies}

Next, line~\ref{alg:answer-query:defun} removes function symbols from the body
Algorithm~\ref{alg:answer-query} and thus prepares the result so that it can be
evaluated using the chase for logic programs. The correctness of this step is
easily verified.

\propdefun*

\begin{proof}
Let ${I_1 = \fixpoint{P_1}{B}}$ and ${I_2 = \fixpoint{P_2}{B}}$, where ${P_1 =
P \cup \Ref{P} \cup \Cong{P} \cup \SymTrans}$ and ${P_2 = P' \cup \Ref{P'} \cup
\SymTrans}$. By routine inductions on the construction of $I_1$ and $I_2$, one
can show that
\begin{displaymath}
    I_2 = I_1 \cup \{ \funpred{f}({\vec t, f(\vec t)}) \mid f(\vec t) \text{ occurs in } I_1 \text{ and } \funpred{f} \text{ occurs in } P' \}.
\end{displaymath}
Indeed, each derivation of a fact $F$ in $I_1$ by a rule ${r \in P_1}$
corresponds precisely to the derivation of $F$ and ${\funpred{f}(\vec t, f(\vec
t))}$ for each term ${f(\vec t)}$ occurring in $F$ in $I_2$ by the rules
obtained from $r$ by transformations in Definition~\ref{def:defun}; moreover,
each term occurring in some fact with predicate of the form $\funpred{f}$ also
occurs in a fact with a predicate not of the form $\funpred{f}$ and so the
reflexivity rules in ${\Ref{P'} \setminus \Ref{P}}$ do not derive any new
facts. Thus, ${P_1 \cup B}$ and ${P_2 \cup B}$ entail the same facts over
predicates occurring in $P$.
\end{proof}

\subsection{Correctness of Desingularization}

Line~\ref{alg:answer-query:desg} of Algorithm~\ref{alg:answer-query}
desingularizes the rules from the preceding steps, as described in Section
\ref{sec:final}, by removing all equalities in the rule bodies. If a rule body
$\gamma$ is desingularizing as $\gamma'$, then clearly $\gamma'$ is a logical
consequence of $\gamma$ and the congruence axioms for equality (i.e.,
axioms~\eqref{eq:congruence} from Section~\ref{sec:preliminaries}). Since we
wish to answer a query $\predQ$ in the presence of congruence axioms, the
desingularized rules clearly return all certain answers. The argument in the
other direction (i.e., that adding congruence axioms does not introduce
incorrect answers) is more subtle, and it requires analyzing the output of our
pipeline as a whole.

\thmdesg*

\begin{proof}
Let $\Sigma_1$ and $P_2$--$P_6$ be as specified in
Algorithm~\ref{sec:motivation}. Set $\Sigma_1$ and programs $P_2$--$P_6$ are
all $\equals$-safe, and thus desingularization correctly produces a program
where each variable in a rule also occurs in the body. Now let ${P' = P_6 \cup
\Ref{P_6} \cup \Cong{P_6} \cup \SymTrans}$ and consider a tuple of constants
${\vec a}$.

\medskip

($\Rightarrow$) Assume ${\Sigma \cup B \modelsEq \predQ(\vec a)}$. Then, ${P_5
\cup \Ref{P_5} \cup \SymTrans \cup B \models \predQ(\vec a)}$ holds by
(i)~Proposition~\ref{prop:sg}, capturing the properties of singularization,
(ii)~the properties of Skolemization in Section~\ref{sec:preliminaries},
(iii)~Theorems~\ref{theorem:relevance} and~\ref{theorem:magic}, capturing the
correctness of relevance analysis and the magic sets transformation, and
(iv)~Proposition~\ref{prop:defun}, capturing the correctness of removal of
function terms from rule bodies. Now let ${P'' = P_5 \cup \Ref{P_6} \cup
\Cong{P_6} \cup \SymTrans}$. Entailment in first-order logic is monotonic as
rules are added; moreover, we clearly have ${\Ref{P_5} = \Ref{P_6}}$ and
${\Cong{P_5} = \Cong{P_6}}$, so ${P'' \cup B \models \predQ(\vec a)}$ holds.
Let ${I_0, I_1, \dots}$ be the sequence of instances used to compute
$\fixpoint{P''}{B}$ as defined in Section~\ref{sec:preliminaries}. By induction
on $i$, we show that ${P' \cup B \models F}$ holds for each fact ${F \in I_i}$,
which implies our claim. The base case is obvious, so assume that this property
holds for $I_i$ and consider an application of a rule ${r \in P''}$ with
substitution $\sigma$ such that ${\sigma(\body{r}) \subseteq I_i}$. Let
$\sigma'$ be obtained from $\sigma$ by setting ${\sigma'(x) = \sigma(t)}$ for
each variable $x$ replaced with term $t$ when Definition~\ref{def:desg} is
applied to $P_5$. For each atom ${R(\vec t) \in \body{r}}$, the inductive
assumption ensures ${P' \cup B \models \sigma(R(\vec t))}$, and the congruence
rules in $\Cong{P_6}$ for $R_i$ clearly ensure ${P' \cup B \models
\sigma'(R(\vec t))}$. Thus, we have ${P' \cup B \models \sigma'(\head{r})}$.
Finally, the transformation from Definition~\ref{def:desg} does not affect
variables in the head of $r$, so we have ${\sigma(\head{r}) =
\sigma'(\head{r})}$.

\medskip

($\Leftarrow$) Let ${P'' = \sk{\Sigma} \cup \Ref{\sk{\Sigma}} \cup
\Cong{\sk{\Sigma}} \cup \SymTrans}$ be the program obtained by Skolemizing
$\Sigma$ and then axiomatizing equality. By the properties of Skolemization
from Section~\ref{sec:preliminaries}, we have ${\Sigma \cup B \modelsEq
\predQ(\vec a)}$ if and only if ${P'' \cup B \models \predQ(\vec a)}$. Let
${I_0, I_1, \dots}$ be the sequence of instances used to compute
$\fixpoint{P'}{B}$ as defined in Section~\ref{sec:preliminaries}. By induction
on $i$, we show that, for each fact ${R(\vec t) \in I_i}$ not of the form ${s
\equals s}$ and where $R$ occurs in $P''$ (i.e., $R$ is not a magic predicate
or of the form $\funpred{f}$), we have ${P'' \cup B \models R(\vec t)}$ . The
base case holds trivially, so assume that this property holds for $I_i$ and
consider an application of a rule ${r' \in P'}$ with substitution $\sigma$ such
that ${\sigma(\body{r'}) \subseteq I_i}$, the predicate of $\head{r'}$ occurs
in $P''$, and $\sigma(\head{r'})$ is not of the form ${s \equals s}$. Note that
$r'$ cannot be a reflexivity axiom (since such axioms always derive ${s \equals
s}$). If $r'$ is a congruence axiom, then we have ${r' \in
\Cong{\sk{\Sigma}}}$, so the property holds by inductive assumption. If ${r'
\in \SymTrans}$ and $\sigma(\head{r'})$ is not of the form ${s \equals s}$,
then $r'$ can derive a new fact only if $\sigma(\body{r'})$ does not contain a
fact of such a form, so the property again holds by the inductive assumption.
The only remaining possibility is ${r' \in P_6}$ with the predicate of
$\head{r}$ occurring in $P''$. Rule $r'$ is obtained by singularizing and then
Skolemizing an existential rule ${\tau \in \Sigma}$, possibly desingularizing
some equalities in the body during relevance analysis in
line~\ref{alg:relevance:UNA:desg} of Algorithm~\ref{alg:relevance}, adding some
magic atom ${\mgc{R}{\alpha}(\vec t^\beta)}$ to the body in
line~\ref{alg:magic:mod-rule} of Algorithm~\ref{alg:magic}, possibly
transforming away the function symbols from this magic atom, and finally
designularizing the remaining equalities. W.l.o.g.\ we can assume that all
desingularization steps ``undo'' the effects of the initial singularization
(i.e., all results of desingularization are unique up to variable renaming).
Moreover, Skolem terms contain only the variables occurring in the rule heads,
and these are not affected by singularization and the other transformations,
which modify only the bodies; thus, the heads of the rules are the same to what
would be produced by Skolemizing $\tau$ before singularization. Hence, there
exists a rule ${r'' \in P''}$ such that ${\head{r''} = \head{r'}}$ and
${\body{r''} \subseteq \body{r'}}$. Since the body of $r''$ does not contain
equality atoms, the inductive assumption ensures ${P'' \cup B \models
\sigma(\body{r''})}$, so we have ${P'' \cup B \models \sigma(\head{r''})}$, as
required.
\end{proof}

\subsection{Final correctness argument}

The final step of Algorithm~\ref{alg:answer-query} is Step
\ref{alg:answer-query:chase}, which simply applies the chase for logic programs
to the output of the prior steps.

\thmalgcorrect*

The correctness claim of Theorem \ref{thm:query-answering} follows from
combining Proposition~\ref{prop:chasecorrect}, the properties of Skolemization
form Section~\ref{sec:preliminaries}, and all the results proved in this
section. For the termination claim, if the chase of $\sg{\Sigma}$ terminates on
every base instance, then line~\ref{alg:relevance:abstraction-fixpoint} of
Algorithm~\ref{alg:relevance} necessarily terminates.

}{}

\end{document}